\documentclass[11pt,english]{article}

\usepackage[cp1251]{inputenc}

\usepackage{amsfonts,amsmath,amsthm,amscd,amssymb,latexsym,
cite,verbatim,texdraw,pb-diagram}
\usepackage{caption}

\usepackage{multirow,array}

 \usepackage[usenames]{color}

\usepackage{graphicx}

\usepackage{booktabs}
\usepackage{multirow}
\usepackage{tabularx}
\usepackage{float}

\theoremstyle{plain}
\newtheorem{remark}{Remark}

\newtheorem{theorem}{Theorem}
\newtheorem{lemma}[theorem]{Lemma}
\newtheorem{proposition}[theorem]{Proposition}
\newtheorem{corollary}[theorem]{Corollary}
\newtheorem{assumption}[theorem]{Assumption}

\renewcommand{\abstract}{\textbf{Abstract. }\medskip}

\newcommand{\diag}{\text{diag}\medskip}

\newcommand{\z}{\mathbf{z}}

\newcommand{\X}{\mathsf{X}}
\newcommand{\F}{\mathcal{F}}
\newcommand{\K}{\mathsf{K}}
\newcommand{\HK}{\mathcal{H}_{\K}}

\newcommand{\Znu}{\z^{\nu}}
\newcommand{\PZnu}{\mathrm{P}_{\z^{\nu}}}

\newcommand{\Po}{\mathrm{P}}

\newcommand{\trace}{\text{trace}\medskip}
\numberwithin{equation}{section} \numberwithin{theorem}{section}
\usepackage{geometry}
\allowdisplaybreaks

\title{
Convergence Analysis of Nystr\"om Subsampling in Covariate Shift Adaptation for Misspecified Case
}
\author{Hanna Myleiko \footnotemark[2]
         \and Sergei Solodky \footnotemark[3] 
         \and Vasyl Semenov \footnotemark[4]  }

\date{}
\begin{document}

\maketitle
\renewcommand{\thefootnote}{\fnsymbol{footnote}}

\footnotetext[2]{ Institute of Mathematics NAS of Ukraine, 3 Tereschenkivska st., Kyiv, Ukraine. Email: hannamyleiko@gmail.com}
\footnotetext[3]{ Institute of Mathematics NAS of Ukraine, 3 Tereschenkivska st., Kyiv, Ukraine; University of Giessen, Department of Mathematics, Giessen, Germany. Email: solodky@imath.kiev.ua}
\footnotetext[4]{ Kyiv Academic University, 36 Vernadsky blvd., Kyiv, Ukraine; Email: vasyl.delta@gmail.com}

\makeatletter
\let\@fnsymbol\@arabic
\makeatother

\maketitle

\begin{abstract}

This paper investigates convergence properties of regularized Nystr\"om subsampling applied to the unsupervised domain adaptation problem under covariate shift. We focus on the low-smoothness (misspecified) case where the target function lies outside the reproducing kernel Hilbert space. By combining Tikhonov regularization with Nystr\"om projection onto a subsampled subspace, we obtain upper bounds on the excess risk that hold with high probability and are expressed in terms of the source condition, the effective dimension, and the sample sizes. We further extend the analysis to the setting where the Radon–Nikodym derivative between the target and source marginal distributions is unknown and must be approximated, and we identify the minimal additional sample sizes required to maintain the same convergence rate as in the oracle case.
\end{abstract}


\section{Introduction} 

The goal of the present paper is to study the regularized Nystr\"om subsampling method for unsupervised domain adaptation under covariate shift when the target  function does not necessarily belong to the reproducing kernel Hilbert space (RKHS).

Recent advances in machine learning have enabled highly accurate solutions to diverse problems ranging from computer vision and natural language processing to speech recognition and recommendation. However, these achievements typically depend on access to abundant, high-quality labeled data, which in practice can be prohibitively costly to obtain. This motivates the transfer learning paradigm, in which a model trained on a label-rich source domain is adapted to perform well on a related target domain where labels are limited or unavailable.

Transferring a source-trained model directly to the target domain generally leads to a significant loss of accuracy because the two domains are governed by different data distributions. This challenge, commonly referred to as domain shift, has spurred active research in domain adaptation — the subfield of machine learning that develops methods for training models on source data so that they generalize effectively to a shifted target distribution.

We work under the covariate shift assumption, which posits that only the marginal distribution of the inputs changes across domains, while the input–output relationship remains the same. Our main contribution is a convergence analysis of the Nystr\"om-regularized estimator under general, low-smoothness source conditions, covering both the case of a known density ratio and the practically important case where the ratio is estimated from finite samples.

Despite its practical relevance, the misspecified setting, in which the target function lies outside the RKHS, has been studied far less extensively than the well-specified case (see, e.g., \cite{PerBook}, \cite{GizewskiMayer21}, \cite{HuangSmola06}, \cite{Shimod00}, \cite{SolMyl24}). Recently, several  studies have attempted to extend the theory beyond the well-specified setting. In the context of the Big Data settings, for the classical supervised learning, where it is typically assumed that the training and test samples are drawn from the same underlying distribution, this problem was addressed in \cite{LuMathePer-Jr}.  Furthermore, domain adaptation problems with covariate shift in the misspecified framework were studied in \cite{Gogal}, \cite{guo}. To the best of our knowledge, domain adaptation under covariate shift in Big Data settings for the misspecified case remains largely underexplored. The present work aims to bridge this gap. To establish convergence guarantees, we employ Tikhonov regularization, while the Nystr\"om subsampling technique is used to mitigate the challenges associated with large-scale datasets.
 
The paper is organized as follows. In Section 2, we give the strict problem settings and define the regularized Nystr\"om subsampling method. In the next section auxiliary statements and assumptions necessary for further research are presented.  In Sections 4 and 5, we obtain error estimates for the regularized Nystr\"om subsampling in the case when the Radon-Nikodym derivative $\beta$ is given, and also when $\beta$ is unknown, correspondingly.  Finally, in Section 7 we presents numerical tests illustrating the theoretical results. Proofs of some results  are given in Appendix.

\section{Problem Settings}
Domain adaptation addresses the challenge of transferring knowledge between domains that follow different data distributions.
Let $x \in X \subset
\mathbb{R}^d$ denote an input variable and $y \in Y\subset
\mathbb{R}$ the corresponding output variable. We assume that observations in the source and target domains are generated according to distinct joint probability measures $p(x,y)$ and $q(x,y)$ defined on $X \times Y$.
In domain adaptation settings, $p(x, y)$ and $q(x, y)$ are usually referred to as the source and target distributions.
In general, the mapping from inputs to outputs is stochastic rather than deterministic. Therefore, an input $x \in X$ does not uniquely determine an output $y\in Y$; instead, the relationship is described by an unknown conditional distribution $\rho(y | x)$ of $y$ given $x$. The inputs themselves are assumed to be random and distributed according to marginal measures $\rho_S(x)$ and $\rho_T(x)$ in the source and target domains, respectively.
Thus, the learning task can be formulated as minimizing the expected risk of predicting $y$ from $x$ under a target distribution  $q(x, y)$, using a training sample
 $\z =\{(x_i, y_i),\ x_i\in X,\ y_i\in Y,\ i=1,2,\ldots,n\},\ |\z|=n$,
drawn i.i.d. over the source distribution
$p(x, y).$  

The covariate shift assumption plays a fundamental role in making domain adaptation feasible (see, e.g., \cite{Shimod00},\cite{HuangSmola06}). Specifically, it assumes
that the conditional distribution  $\rho(y | x)$ remains invariant across domains, while the marginal distribution of the inputs changes, i.e.,  $\rho_S(x)\neq \rho_T(x).$ Under the covariate shift assumption, the source and target joint distributions can be factorized as
\begin{equation}\label{eq:fact}
p(x, y)=\rho(y | x)\rho_S(x),\quad q(x, y)=\rho(y | x)\rho_T(x).
\end{equation}

In the present study, we restrict ourselves to learning with the least square loss where the expected risk
of the prediction  $y$ from $x$ by means of a function $f\colon X\rightarrow Y$ is defined in the target domain as follows
$$
\mathcal{R}_q(f):=\int_{X\times Y}(f(x)-y)^{2}dq(x,y),
$$
which is minimized by so-called regression function
\begin{equation}\label{eq:fq}
f(x)=f_{q}(x) = \int_{Y} y d \rho(y | x) .
\end{equation}
However, in the unsupervised domain adaptation setting, neither $\mathcal{R}_q(f)$ nor $f_q(x)$ can be computed directly, since the underlying distribution $q(x,y)$ is not available. Instead, the target domain information is provided only through a set of unlabeled samples $X^{'}=(x_1^{'},x_2^{'},\ldots,
x_m^{'}),\quad |X^{'}|=m,$ 
where $x_i'$ are drawn i.i.d. from the target marginal distribution $\rho_T(x)$. Therefore, the goal is  to leverage this unlabeled target data together with the labeled training dataset $\mathbf{z}$ to construct an empirical estimator $f_{\mathbf{z}}$ that approximates the ideal minimizer $f_q$ in the term of excess risk
$$
\mathcal{R}_q(f_{\z})-\mathcal{R}_q(f_{q}):=\|f_{\z}-f_q\|^{2}_{L_{2,\rho_T}},
$$
where $L_{2,\rho_T}$ is the space of square integrable functions $f\colon X\rightarrow\mathbb{R}$ with respect to the marginal probability measure $\rho_T$.
Following \cite{GizewskiMayer21}, we assume that unsupervised domain adaptation and standard supervised learning aim to approximate the same regression function given by (\ref{eq:fq}).

As we mentioned earlier, most existing theoretical analysis has been carried out primarily in the well-specified case, when the regression function $f_q$ belongs to a some Reproducing Kernel  Hilbert Space (RKHS) $\HK$. 
In our research we focused on the case when the regression function $f_q$ minimizing expected risk 
 $\mathcal{R}_q(f)$ belongs to the space $L_{2,\rho_T} \backslash \HK.$ 
 Let $J_T : \mathcal{H}_\K \hookrightarrow L_{2, \rho_T}$ and $J_S : \mathcal{H}_\K \hookrightarrow
L_{2, \rho_S}$ be the inclusion operators. Recall that the information about the source and the target marginal distributions are only provided in the form of samples  $X_S=\{x_1,x_2,\ldots, x_n\}$ and $X_T=\{x_1^{'},x_2^{'},\ldots, x_m^{'}\}$, drawn 
i.i.d. from $\rho_S$ and $\rho_T$, respectively. In the sequel, we define two sample operators 
$$
S_{X_T}f=(f(x_1^{'}),f(x_2^{'}),\ldots, f(x_m^{'}))\in\mathbb{R}^m,
$$
$$
S_{X_S}f=(f(x_1),f(x_2),\ldots,f(x_n))\in\mathbb{R}^n,
$$
acting from $\mathcal{H}_\K $ to $\mathbb{R}^m$ and $\mathbb{R}^n$, where the norms in  later spaces are 
$m^{-1}$-times and $n^{-1}$-times the standard Euclidian norms, such that the adjoint operators  $S^{*}_{X_T}\colon
\mathbb{R}^m\rightarrow \mathcal{H}_\K$ and $S^{*}_{X_S}\colon \mathbb{R}^n\rightarrow \mathcal{H}_\K$ are given as
$$
S^{*}_{X_T}u(\cdot)=\frac{1}{m}\sum_{j=1}^{m}\K(\cdot,x_{j}^{'})u_j,\quad\quad u=(u_1,u_2,\ldots,u_m)\in\mathbb{R}^m,
$$
$$
S^{*}_{X_S}v(\cdot)=\frac{1}{n}\sum_{i=1}^{n}\K(\cdot,x_{i})v_i,\quad\quad v=(v_1,v_2,\ldots,v_n)\in\mathbb{R}^n.
$$
Since we have no direct access to both the target probability measure $\rho_T$ and the space $L_{2,\rho_T}$ in which we are going to approximate the regression function $f^{*}=f_q$, then an assumption should be put  on the relation between the source probability $\rho_S$ and the target probability $\rho_T$.

As in \cite{HuangSmola06}, we assume that the target measure $\rho_T$
is absolutely continuous with respect to the source measure $\rho_S$.
Hence, there exists a nonnegative function $\beta\colon X\rightarrow\mathbb{R}_{+}$ such that
$$
d\rho_T(x)=\beta(x)d\rho_S(x).
$$
The function $\beta(x)$ is considered as the Radon-Nikodym derivative $\frac{d\rho_{T}}{d\rho_S}$ of the target probability measure with respect to the source one.

 We also assume that we only have access to  the values  $\beta(x_i)$  of the Radon-Nikodym derivative $\beta(x)=\frac{d\rho_{T}}{d\rho_S}$
at the points $x_i,\, i=1,2,\ldots,n$,
 drawn i.i.d. from $\rho_S(x)$ and  we consider a diagonal $n\times n$ matrix
 $B=\diag(\beta(x_1),\beta(x_2),\ldots,\beta(x_n)).$
Moreover, we assume  that $\beta(x)$ is uniformly bounded on  $X$, such that $0\le \beta(x)\le b_0$ for some
 $b_0>0$ and any $x\in X.$

The subsequent analysis bases on two additional assumptions which  are  common and not restrictive.   We assume that $\K\colon X\times X\rightarrow \mathbb{R}$  is a continuous and bounded  kernel that for any  $x\in X$ it holds
 \begin{eqnarray}\label{ker}
 \|\K(\cdot,x)\|^{2}_{\HK}=\left<\K(\cdot,x) \K(\cdot,x)\right>_{\HK}=\K(x,x)\le\kappa<\infty.
 \end{eqnarray}
 In addition, we assume that for any input $x\in X$  corresponding output 
 $y\in Y\subset \mathbb{R}$ is bounded $|y|\le y_0$ with $y_0>0.$
 
Further, we are going to approximate a solution of the equation arising from the minimization of the excess risk

\begin{equation}\label{eq:exrisk}
\mathcal{R}_{q}(f)-\mathcal{R}_{q}(f_q)=\|f-f_q\|_{L_{2},\rho_{T}}^{2}.
\end{equation}

The minimization problem (\ref{eq:exrisk}) can be rewritten by means of the inclusion operator 
$J_T\colon \HK\hookrightarrow L_{2,\rho_T}$ as a variational problem
$$
\|J_Tf-f_q\|_{L_{2},\rho_{T}}\rightarrow \min,
$$
and it leads to the equation
\begin{equation}\label{eq:eqnorm}
J_T f=f_q.
\end{equation}

Note that the equation (\ref{eq:eqnorm}) is ill-posed since (\ref{eq:eqnorm})  is the Fredholm integral equation of the first kind with compact operator $J_T.$\\

{\bf{Source condition.}}
Let $\mathcal{H}$ be a Hilbert space  
and  $T\colon \mathcal{H}\rightarrow\mathcal{H},\, \|T\|_{\mathcal{H}\rightarrow\mathcal{H}}\le l, \, l>0,$
be a compact, injective, self-adjoint, and non-negative linear operator. For every $h\in\mathcal{H}$ there is a continuous, strictly increasing function $\varphi\colon [0,l] \rightarrow \mathbb{R}$, such that  $\varphi(0)=0$ and $\varphi^2$ is concave. The set of all such functions we denote as $\F(\mathcal{H}).$ 
If $h\in\mathcal{H}$ it can be presented  as 
\begin{eqnarray}\label{source_cond}
h=\varphi(T)\mu,\quad\quad \|\mu\|_{\mathcal{H}}\le \varkappa,
\end{eqnarray}
where $\varkappa>0,$ then the expression (\ref{source_cond}) is usually  called  "source condition"$\,$  and $\varphi$ is the index function of the source condition (see, e.g., \cite{LuPer}).

This function specifies a smoothness properties of $h$ and describes the convergence rate of the regularization method. 
It is known that in the low-smoothness case  the function $t\longmapsto \sqrt{t}/\varphi(t) $ is nondecreasing. It means that the misspecified case studied here corresponds to (\ref{source_cond}) with $\varphi(t)$ increasing not faster than $\sqrt{t}.$

It is known that for functions $\varphi\in \F(\mathcal{H})$ to guarantee  the optimal order of accuracy, it is sufficient to apply  the standard Tikhonov method with a generating function of the form $g_{\lambda}(t)=(t+\lambda)^{-1},\, t,\lambda>0.$ Recall, that for the generating function of the Tikhonov regularization the following statements are known (see, e.g., \cite{VV}, \cite{LuPer}, \cite{MP2003}):\\ 
there are positive constants $\gamma_{0}, \overline{\gamma}, \tilde{\gamma},$ and $\gamma_{p}$ such that
\begin{equation}\label{eq:g}
\underset{0<t\leq l}{\sup}\,|1-t(t+\lambda)^{-1}|\leq\gamma_{0},\qquad
\underset{0<t\leq
l}{\sup}\,\sqrt{t}|(t+\lambda)^{-1}|\leq\frac{\overline{\gamma}}{\sqrt{\lambda}},\qquad
\underset{0<t\leq
l}{\sup}\,|(t+\lambda)^{-1}|\leq\frac{\tilde{\gamma}}{\lambda},
\end{equation}
\begin{equation}\label{eq:gp}
\underset{0<t\leq l}{\sup}\, t^{p}|1-t(t+\lambda)^{-1}|\leq
\gamma_{p}\lambda^{p} ,
\end{equation}

\begin{equation}\label{cond_qualific} 
\underset{0<t\leq l}{\sup}\,|1-t(t+\lambda)^{-1}|\varphi(t) \leq
\gamma_{\ast} \varphi(\lambda) ,\quad \gamma_{\ast} = \max\{\gamma_{0}, \gamma_{p}\};
\end{equation}
 
\begin{equation}\label{qualific_root} 
\underset{0<t\leq l}{\sup}\,|1-t (t+\lambda)^{-1}|\sqrt{t}\varphi(t) \leq
\gamma_{*} \sqrt{\lambda}\phi(\lambda).
\end{equation}

It is worth noting (see, e.g., \cite{LitPMDiscrStr}) that any $\varphi\in \F(\mathcal{H})$ is an operator monotone function. In particularly,  when $\mathcal{H}=L_{2,\rho_T},$ then for any non-negative self-adjoint operators 
$A,B\colon L_{2,\rho_T}\rightarrow L_{2,\rho_T}$ with spectra in $[0,l]$ it holds
\begin{equation} \label{mon_op}
\|\varphi(A)-\varphi(B)\|_{L_{2,\rho_T}\rightarrow L_{2,\rho_T}}\leq c_1
\varphi\left(\|A-B\|_{L_{2,\rho_T}\rightarrow L_{2,\rho_T}}\right).
\end{equation}

{\bf{Nystr\"om subsampling.}}\\
The Nystr\"om type subsampling provides an efficient strategy to conquer the big data challenges. This technique consists of the methods replacing  the entire kernel matrix by a smaller matrix of significantly lower rank, obtained by a random  columns subsampling. 
It is known (see, e.g., \cite{RudiComRos15}) that the Nystr\"om subsampling can be considered as a combination of the standard Tikhonov regularization and a projection scheme on the subset

\begin{eqnarray}\label{HK}
\HK^{\Znu}\colon=\left\{f\colon f=\sum_{i=1}^{|\Znu|}\,c_i\K(\cdot,x_i)+\sum_{j=1}^{|\Znu|}\,c^{'}_j\K(\cdot,x^{'}_j)\right\}, 
\end{eqnarray}
where $|\Znu|\ll \min\{n,m\}.$

Thus, the approximation to $f_q$ we will seek as follows 
\begin{equation}\label{eq2}
f_{\z,\Znu}^{\lambda_{m,n}} =  \left( \lambda I+P_{\z^{\nu}}S_{X_S}^{\ast} BS_{X_S}P_{\z^{\nu}} \right)^{-1} \PZnu S_{X_S}^{\ast} B\overline{y},
\end{equation}
where $\Po_{\Znu}: \HK\rightarrow\HK^{\Znu}$,
$\|\PZnu\|_{\HK\rightarrow\HK}=1$ is the orthogonal projection operator with the range $\HK^{\Znu}$ and $\overline{y}=(y_1,y_2,\ldots,y_n)$ is the vector of outputs corresponding to inputs $\X_{S}=\{x_1,x_2,\ldots,x_n\}$.
Note  (see, e.g, \cite{RudiComRos15}), to compute (\ref{eq2}) it is not necessary to construct   $\Po_{\Znu}$ explicitly.
\\ \\
{\bf{Effective dimension.}}
\\
For each $\lambda>0$ it holds
$$
\mathcal{N}_{x}(\lambda):=\left\langle \K(\cdot,x),(\lambda I+J_T^{\ast}J_T)^{-1} \K(\cdot,x)\right\rangle_{\HK}=\|(\lambda I+J_T^{\ast}J_T)^{-\frac{1}{2}}\K(\cdot,x)\|^2_{\HK}< \infty.
$$
We highlight the related quantities
\begin{eqnarray}\label{N_inf}
\mathcal{N}_{\infty}(\lambda):=\sup_{x\in X} \mathcal{N}_{x}(\lambda)
\end{eqnarray}
and
\begin{eqnarray}\label{N_lambda}
\mathcal{N}(\lambda):=\int_{X}\mathcal{N}_{x}(\lambda) d\rho_T(x)=\trace\{(\lambda I+J_T^{\ast}J_T)^{-1} J_T^{\ast}J_T\}.
\end{eqnarray}
The function $\mathcal{N}$ measures the capacity of the RKHS $\HK$ in the space $L_{2,\rho_T}$ and it is called the effective dimension. 

\section{Auxiliary statements and assumptions}  
In this section we provide some auxiliary statements and assumptions that will be used in the proofs in next sections.
\begin{assumption}\label{source_f}
Assume that  $f_q\in L_{2,\rho_{T}}/ \HK$ satisfies the source condition 
(\ref{source_cond}) with $\varphi\in\mathcal{F}(L_{2,\rho_{T}})$ and $T=J_{T}J^{*}_{T},$ then
\begin{eqnarray}\label{source_cond_f}
f_q=\varphi(J_{T}J^{*}_{T})\mu_{q},
\end{eqnarray}
where $\|\mu_{q}\|_{L_{2},\rho_{T}}\le \overline{\varkappa}, \, \overline{\varkappa}>0.$\\
\end{assumption}
\begin{assumption}\label{source_ker} Let index function
 $\zeta\in\mathcal{F}(\HK)$, $T=J^{*}_{T}J_{T}$,  and $\zeta^{2}$ is covered by qualification $p=1$ such that, for all $x\in \X,$
\begin{equation}\label{ker_source}
\K(\cdot,x)=\zeta(J^{*}_{T}J_{T})\mu_{\K},
\end{equation}
where $\|\mu_{\K}\|_{\HK}\le \hat{\varkappa},\, \hat{\varkappa}>0$  does not depend on $x.$
\end{assumption}
Note that the condition (\ref{ker_source}) is the source condition for kernel section $\K(\cdot,x).$ As before, we will consider such $\zeta$ which allows a representation in the power scale $t^r,\, 0<r\leq\frac{1}{2},$ as well as all less smooth ones.  
\begin{assumption}\label{source_beta}
Assume that  $\beta=\frac{d\rho_T}{d\rho_S}\in\HK$ satisfies the source condition 
(\ref{source_cond}) with $\psi\in\mathcal{F}(\HK)$ and $T=J^{*}_{S}J_{S},$ then
\begin{eqnarray}\label{source_cond_beta}
\beta=\psi(J_{S}^{*}J_S)\mu_{\beta},
\end{eqnarray}
where $\|\mu_{\beta}\|_{\HK}\le \tilde{\varkappa}, \, \tilde{\varkappa}>0.$
\end{assumption}
Here and in the sequel, we adopt the convention that $C$ denotes a generic positive coefficient, which can vary from inequality to inequality and does not depend on the values of $n,m,\lambda,$ and $\delta.$

\begin{lemma}\label{lem:dem}\cite[Lemma 5]{NgPer23}
Under Assumption \ref{source_ker}, it holds
$$
\mathcal{N}_{\infty}(\lambda)\le C\frac{\zeta^{2}(\lambda)}{\lambda}.
$$
\end{lemma}

For operator $J_{T}\colon \HK\hookrightarrow L_{2,\rho_{T}}$, \cite[Lemma 3.3]{SolMylRecov25} ensure that

\begin{eqnarray}\label{lem:polar}
\|J_{T}(\lambda I + J_{T}^{*}J_{T})^{-\frac{1}{2}}\|_{\HK\rightarrow L_{2,\rho_{T}}}
\le 1.
\end{eqnarray}

Further, we give the following relation for the regularization parameter $\lambda>0,$ sample size $n, $ and the subsample size $|\Znu|.$
For $0<\delta<1$, with probability at least $1-\delta$, we require that
\begin{equation}\label{sample_choice}
|\z^{\nu}|\ge C\mathcal{N}_{\infty}(\lambda)\log\frac{1}{\lambda}\log\frac{1}{\delta}
\end{equation}
and
\begin{equation}\label{par_choice}
\lambda \in \left[C n^{-1}\log\frac{n}{\delta}, l\right].
\end{equation}
If $\Znu$ is subsampled according to the plain Nystr\"om approach, then (see, e.g.,\cite[Lemma 6]{RudiComRos15}, \cite[Corollary 1]{MylSolPer-Jr19}) with probability at least $1-\delta$ it holds  
\begin{eqnarray}\label{eq:prob10}
\|(I-\PZnu)(\lambda I+J_T^{*}J_T)^{1/2}\|^{2}_{\HK\rightarrow \HK}\le 3\lambda,
\end{eqnarray}
\begin{eqnarray}\label{eq:prob1}
\|J_T(I-\PZnu)\|^{2}_{\HK\rightarrow L_{2,\rho_{T}}}\le 3\lambda,
\end{eqnarray}
and   for any $\varphi\in\F(\HK)$ (see \cite[Proposition 2]{LitPMDiscrStr})  it holds
\begin{eqnarray}\label{saturat}
\|(I - P_{\z^{\nu}})\varphi(J_T^{*}J_T)\|_{\HK\rightarrow\HK}\leq C
\varphi(\|(J_T^{*}J_T)^{1/2}(I - P_{\z^{\nu}})\|^{2}_{\HK\rightarrow\HK} ).
\end{eqnarray}

\begin{lemma}\cite[Lemma 3.4]{SolMylRecov25}\label{lem:add0}
For every choice $\Znu$ from the sample $\X_{S}$ we have that
\begin{eqnarray}\label{eq:prob2}
\|(\lambda I+S_{\X_{S}}^{*}BS_{\X_{S}})^{\frac{1}{2}}\PZnu(\lambda I+\PZnu S_{\X_{S}}^{*}BS_{\X_{S}}\PZnu)^{-1}\PZnu(\lambda I+S_{\X_{S}}^{*}BS_{\X_{S}})^{\frac{1}{2}}\|_{\HK\rightarrow\HK}\le 1,
\end{eqnarray}
\begin{eqnarray}\label{eq:prob20}
\|(\lambda I+S_{\X_{S}}^{*}BS_{\X_{S}})^{\frac{1}{2}}(\lambda I+\PZnu S_{\X_{S}}^{*}BS_{\X_{S}}\PZnu)^{-1}\PZnu(\lambda I+S_{\X_{S}}^{*}BS_{\X_{S}})^{\frac{1}{2}}\|_{\HK\rightarrow\HK}\le 1.
\end{eqnarray}
\end{lemma}

We recourse to the following inequality from \cite[Lemma 3.5]{SolMylRecov25} which assert that
for any $\varphi\in \F(L_{2,\rho_T})$ it holds
\begin{equation}\label{norm_forL2}
\|(\lambda I+J_TJ_T^{*})^{-1/2}\varphi(J_TJ_T^{*})\|_{L_{2,\rho_{T}}\rightarrow L_{2,\rho_{T}}}\le\frac{1}{\sqrt{\lambda}}\varphi(\lambda).
\end{equation}

Following \cite{NgPer23},\cite{LuMathePer}, we introduce the supplemental functions
 \begin{eqnarray}\label{eq:B}
 \mathcal{B}_{m,\lambda}:=\frac{2\kappa}{\sqrt{m}}\left(\frac{\kappa}{\sqrt{m\lambda}}+\sqrt{\mathcal{N}(\lambda)}\right),
 \end{eqnarray}
  \begin{eqnarray}\label{eq:G}
  \mathcal{G}(\lambda):=\left(\frac{\mathcal{B}_{m,\lambda}}{\sqrt{\lambda}}\right)^2+1.
  \end{eqnarray}
  In the sequel, we give the following statement
\begin{lemma}\cite[Lemma 4.6]{LuMathePer}\label{lem:LuPer}
There exists $\lambda_{*}$ such that $\frac{\mathcal{N}(\lambda_{*})}{\lambda_{*}}=m.$ 
For 
\begin{equation}\label{lambda_*}
\lambda_{*}\le\lambda\le\kappa
\end{equation}
 there holds
\begin{equation}\label{est:B}
\mathcal{B}_{m,\lambda}\le\frac{2\kappa}{\sqrt{m}}(\sqrt{2}\kappa+\sqrt{\mathcal{N}(\lambda)}).
\end{equation}
This yields
\begin{equation}\label{est:G}
\mathcal{G}\le1+(4\kappa^2+2\kappa)^{2}
\end{equation}
and also 
\begin{equation}\label{est:B2}
\mathcal{B}_{m,\lambda}(\mathcal{B}_{m,\lambda}+\sqrt{\lambda})\le(1+4\kappa)^{4}\min\left\{\lambda,\sqrt{\frac{\kappa}{m}}\right\}.
\end{equation}
\end{lemma}
\begin{lemma}\cite[Lemma 1]{GizewskiMayer21}\label{lem:perturb}
With probability at least $1-\delta$ it holds
 \begin{eqnarray}\label{perturb_cond}
\begin{split}
\|S_{\X_{T}}^{*}S_{\X_{T}}-S_{X_S}^{\ast} BS_{X_S}\|_{\HK\rightarrow\HK}&\leq C\log^{\frac{1}{2}}\frac{2}{\delta}(m^{-\frac{1}{2}}+n^{-\frac{1}{2}}),\\
\|S_{X_S}^{\ast} BS_{X_S}f_q-S_{X_S}^{\ast} B\overline{y}\|_{\HK}&\le C\log^{\frac{1}{2}}\frac{2}{\delta}(m^{-\frac{1}{2}}+n^{-\frac{1}{2}}).
\end{split}
\end{eqnarray}
\end{lemma}
In what follows, we will use the auxiliary estimates (see, e.g., \cite{NgPer23})
\begin{eqnarray}\label{bound:1}
\|(\lambda I + J_{T}^{*}J_{T})^{-\frac{1}{2}}(J_{T}^{*}J_{T}-S_{\X_{T}}^{*}S_{\X_{T}})\|_{\HK\rightarrow\HK}\le\mathcal{B}_{m,\lambda}\log\frac{2}{\delta},
\end{eqnarray}
\begin{eqnarray}\label{bound:2}
\|(\lambda I + J_{T}^{*}J_{T})(\lambda I+S_{\X_{T}}^{*}S_{\X_{T}})^{-1}\|_{\HK\rightarrow\HK}\le 2\left[\left(\frac{\mathcal{B}_{m,\lambda}\log\frac{2}{\delta}}{\sqrt{\lambda}}\right)^{2}+1\right],
\end{eqnarray}
\begin{eqnarray}\label{bound:3}
\|(\lambda I+S_{\X_{T}}^{*}S_{\X_{T}})(\lambda I + J_{T}^{*}J_{T})^{-1}\|_{\HK\rightarrow\HK}\le \frac{\mathcal{B}_{m,\lambda}\log\frac{2}{\delta}}{\sqrt{\lambda}}+1.
\end{eqnarray}

\begin{proposition}\label{Cordes}\cite[Proposition 4 (Cordes Inequality)]{RudiComRos15}
Let $A,B$ be two positive self-adjoint bounded operators on a separable Hilbert space $\mathcal{H}$. Then for all $0\le s\le 1$ it holds
$$
\|A^{s}B^{s}\|_{\mathcal{H}\rightarrow\mathcal{H}}\le\|AB\|_{\mathcal{H}\rightarrow\mathcal{H}}^{s}.
$$
\end{proposition}

\begin{proposition}\label{concentrat}\cite[Proposition 2]{CapVito}
Let $(\Omega, \mathcal{F}, P)$ be a probability space and let $\xi$ be a random variable on $\Omega$ taking value in real separable Hilbert space $H.$ Assume that there are two positive constants $L$ and $\sigma$ such that 
 \begin{eqnarray}\label{concentrat1}
 \mathbb{E}\|\xi-\mathbb{E}\xi\|^{p}_{H}\le \frac{1}{2} p!\sigma^2 L^{p-2}, 
 \end{eqnarray}
for any $p\ge 2.$ Then for all $ l \in \mathbb{N} $ with probability at least $1-\delta$
it holds
 \begin{eqnarray}\label{concentrat2}
 \|\frac{1}{l} \sum_{i=1}^{l}\, \xi(\omega_i)-\mathbb{E}\xi\|_{H}\le 2\left( \frac{L}{l}+\frac{\sigma}{\sqrt{l}}\right)\log\frac{2}{\delta}. 
 \end{eqnarray}
\end{proposition}
In particular, (\ref{concentrat1}) holds if 
\begin{eqnarray}\label{concentrat3}
\begin{split}
\| \xi(\omega)\|_{H}&\le \frac{L}{2},\quad
\mathbb{E}\|\xi\|_{H}^{2}&\le\sigma^2. 
\end{split}
\end{eqnarray}

\begin{lemma}\label{hu:06}\cite[Lemma 4]{HuangSmola06}
Let $b_0>0$ be such that $|\beta(x)|\le b_0$ for every $x\in \X$ and let $\psi$ be a map from $\X$ into $\HK$ such that $\|\psi(x)\|_{\HK}\le R$ for all $x\in \X.$ Then with probability at least $1-\delta$ it holds
$$\|\frac{1}{m}\sum_{j=1}^{m}\psi(x_{j}^{'})-\frac{1}{n}\sum_{i=1}^{n}\beta(x_i)\psi(x_{i})\|_{\HK}\le (1+\sqrt{2\log\frac{2}{\delta}})R\sqrt{\frac{b_{0}^{2}}{n}+\frac{1}{m}}.
$$
\end{lemma}

\begin{lemma}\label{my:lem2}
For any $0<\delta<1$, with probability at least $1-\delta$, it holds
\begin{eqnarray}\label{eq_my2}
\begin{split}
\|\left(\lambda I+J^{\ast}_{T}J_{T}\right)^{-\frac{1}{2}}\left(S_{X_S}^{\ast} BS_{X_S}-S_{X_{T}}^{\ast}S_{X_{T}}\right)\|_{\HK\rightarrow\HK}\le C\sqrt{\mathcal{N}_{\infty}(\lambda)}\left(\frac{1}{\sqrt{n}}+\frac{1}{\sqrt{m}}\right)\log\frac{2}{\delta}.
\end{split}
\end{eqnarray}
\end{lemma}
\begin{lemma}\label{my:lem1}
For any $0<\delta<1$, with probability at least $1-\delta$, it holds
 \begin{eqnarray}\label{eq_my1}
 \begin{split}
\|(\lambda I&+J_{T}^{*}J_{T})^{\frac{1}{2}}(\lambda I + S_{\X_{S}}^{*}BS_{\X_{S}})^{-\frac{1}{2}}\|_{\HK\rightarrow\HK}\\
&\le C \left(\left(\frac{\mathcal{B}_{m,\lambda}\log\frac{2}{\delta}}{\sqrt{\lambda}}+\sqrt{\frac{\mathcal{N}_{\infty}(\lambda)}{\lambda}}\left(\frac{1}{\sqrt{n}}+\frac{1}{\sqrt{m}}\right)\log\frac{2}{\delta}+\frac{1}{2}\right)^{2}+\frac{3}{4}\right)^{\frac{1}{2}}.
\end{split}
\end{eqnarray}
\end{lemma}

\begin{lemma}\label{my:lem3}
For any $0<\delta<1$, with probability at least $1-\delta$, it holds
\begin{eqnarray}\label{eq_my3}
\begin{split}
\|\left(\lambda I+J^{\ast}_{T}J_{T}\right)^{-\frac{1}{2}}
\left [J_{T}^{\ast}f_q-S_{X_S}^{\ast} B\overline{y}\right]\|_{\HK}\le C \sqrt{\mathcal{N}_{\infty}(\lambda)}\frac{1}{\sqrt{n}}\log\frac{1}{\delta}.
\end{split}
\end{eqnarray}
\end{lemma}

The proofs of Lemmas  \ref{my:lem2}, \ref{my:lem1},  and \ref{my:lem3}  are given in Appendix \ref{ap:a}.

\section{Accuracy estimate of the Nystr\"om subsampling }
In this section, we estimate the approximation accuracy of $f_q$ by means of the Nystr\"om subsampling (\ref{eq2}).

For convenience, we introduce the following quantity:
\begin{eqnarray}\nonumber
\begin{split}
\mathcal{D}^{\lambda}_{m,n}&=\mathcal{D}(\lambda, m,n):=\left(\frac{\mathcal{B}_{m,\lambda}\log\frac{2}{\delta}}{\sqrt{\lambda}}+\sqrt{\frac{\mathcal{N}_{\infty}(\lambda)}{\lambda}}\left(\frac{1}{\sqrt{n}}+\frac{1}{\sqrt{m}}\right)\log\frac{2}{\delta}+\frac{1}{2}\right)^{2}+\frac{3}{4}.
\end{split}
\end{eqnarray}

\begin{lemma}\label{my:lem4}
 If $f_q$ satisfies Assumption \ref{source_f} then for any $|\Znu|$  satisfies (\ref{sample_choice}) it holds
\begin{equation}\label{sig_22}
\| \left(\lambda I+P_{\z^{\nu}}J_{T}^{\ast}J_{T}P_{\z^{\nu}}\right)^{-1}P_{\z^{\nu}}J_{T}^{\ast}f_q\|_{\HK} \le \frac{C}{\sqrt{\lambda}}\varphi(\lambda).
\end{equation}
\end{lemma}
The proof of Lemma \ref{my:lem4} is given in Appendix \ref{ap:a}.
\begin{proposition}\label{resultHK}
Assume that in the plain Nystr\"om subsampling the values $|\Znu|$ and $\lambda$ satisfy (\ref{sample_choice}), (\ref{par_choice}), correspondingly. If $f_q$ satisfies Assumption \ref{source_f}, then with probability at least $1-\delta$ it holds 

 \begin{eqnarray}\label{err_gen}\nonumber
\begin{split}
&\|f_q-J_Tf_{\z,\Znu}^{\lambda_{m,n}}\|_{L_{2},\rho_{T}} \le C\varphi(\lambda)+C\mathcal{D}^{\lambda}_{m,n}\frac{\mathcal{B}_{m,\lambda}}{\sqrt{\lambda}}\varphi(\lambda)\log\frac{2}{\delta}\\
&+ C\mathcal{D}^{\lambda}_{m,n}\sqrt{\frac{\mathcal{N}_{\infty}(\lambda)}{\lambda}}\left(\frac{1}{\sqrt{n}}+\frac{1}{\sqrt{m}}\right)\varphi(\lambda)\log\frac{2}{\delta}
+\mathcal{D}^{\lambda}_{m,n}
\sqrt{\mathcal{N}_{\infty}(\lambda)}\frac{1}{\sqrt{n}}\log\frac{2}{\delta}.
\end{split}
\end{eqnarray}
where $f_{\z,\z^{\nu}}^{\lambda_{m,n}}$ is of the form (\ref{eq2}). 
\end{proposition}

The proof of Proposition \ref{resultHK} is deferred to Appendix \ref{ap:b}. \\\\
For $\lambda>\lambda^{*}$ (see Lemma \ref{lem:LuPer}) we can make the statement of Proposition \ref{resultHK} more transparent.

\begin{theorem}\label{main_HK}
Let $\K$ satisfies Assumption \ref{source_ker} and  $\lambda>\lambda^{*}$. Then, under the assumptions of Proposition \ref{resultHK}, the following holds with probability at least $1-\delta$,
 \begin{eqnarray}\nonumber
\begin{split}
\|f_q-J_Tf_{\z,\Znu}^{\lambda_{m,n}}\|_{L_{2},\rho_{T}} 
&\le C\varphi(\lambda)\left(1+\frac{\zeta(\lambda)}{\lambda}\left(\frac{1}{\sqrt{n}}+\frac{1}{\sqrt{m}}\right)\right)\log\frac{2}{\delta}.
\end{split}
\end{eqnarray}
\end{theorem}
$ $

Based on Theorem \ref{main_HK}, the regularization parameter $\lambda$ can be chosen according to the following rule:

 $$
 \frac{\zeta(\lambda)}{\lambda}\left(\frac{1}{\sqrt{n}}+\frac{1}{\sqrt{m}}\right)=1
 $$
 or
 \begin{equation}\label{xi_eq1}
 \frac{\lambda}{\zeta(\lambda)}=\frac{1}{\sqrt{n}}+\frac{1}{\sqrt{m}}.
\end{equation}
 Note that due to monotonicity of $\zeta$, it is obvious that the equation (\ref{xi_eq1}) always has a unique solution. \\
 Let $\lambda=\lambda_0$ is a solution of (\ref{xi_eq1})
 then under the conditions of Theorem \ref{main_HK}, with probability at least $1-\delta,$ we have
 $$
 \|f_q-J_Tf_{\z,\Znu}^{\lambda_{m,n}}\|_{L_{2},\rho_{T}} \le C\varphi(\lambda_0)\log\frac{2}{\delta}.
 $$
 
 \begin{corollary}\label{col_1} If $f_q$ satisfies Assumption \ref{source_f} with $\varphi(t)=t^{\gamma}$  and $\K$ satisfies Assumption \ref{source_ker} with $\zeta(t)=t^{r},$ where $\gamma, r \in (0,\frac{1}{2}],$ then
  \begin{equation}\label{pover_func1}
 \|f_q-J_Tf_{\z,\Znu}^{\lambda_{m,n}}\|_{L_{2},\rho_{T}}\le C\left(\frac{1}{\sqrt{n}}+\frac{1}{\sqrt{m}}\right)^{\frac{\gamma}{1-r}}\log\frac{2}{\delta}.
 \end{equation}
 \end{corollary}
 \begin{proof}
 From (\ref{xi_eq1}) it follows
$
\lambda_{0}^{1-r}=\left(\frac{1}{\sqrt{n}}+\frac{1}{\sqrt{m}}\right) $ or
$  \lambda_{0} =\left(\frac{1}{\sqrt{n}}+\frac{1}{\sqrt{m}}\right)^{\frac{1}{1-r}}.$
According to Theorem \ref{main_HK} we obtain
 $$
 \|f_q-J_Tf_{\z,\Znu}^{\lambda_{m,n}}\|_{L_{2},\rho_{T}}\le C\left(\frac{1}{\sqrt{n}}+\frac{1}{\sqrt{m}}\right)^{\frac{\gamma}{1-r}}\log\frac{2}{\delta}.
 $$
\end{proof}
 \section{The case of unknown the Radon-Nikodym derivative}
 \subsection{Main result}
 In what follows, we assume that the exact values of the Radon-Nikodym derivative $\beta=\frac{d\rho_T}{d\rho_S}$ are unknown. In the context of such setting,
  we will deal with the matrix $$B_{M,N}=\diag(\beta_{M,N}^{\lambda_{M,N}}(x_1),\beta_{M,N}^{\lambda_{M,N}}(x_2),\ldots,\beta_{M,N}^{\lambda_{M,N}}(x_n))$$
  instead of the matrix $B$. Here, we assume that  sizes $M$ and $N$  of i.i.d. samples $(x_{1}^{'},x_{2}^{'},\ldots, x_{M}^{'})$ and $(x_1,x_2,\ldots, x_N)$ drawn respectively from $\rho_T$ and $\rho_S$ are much larger then $m$ and $n$ appearing in Theorem \ref{resultHK}.
Then the approximate solution (\ref{eq2}) has the form
\begin{eqnarray}\label{app_BMN}
f^{\lambda_{M,N}}_{\z,\Znu}=(\lambda I+\PZnu S_{X_S}^{*}B_{M,N}S_{X_S}\PZnu )^{-1}\PZnu S_{X_S}^{*}B_{M,N}\overline{y}.
\end{eqnarray}
Note (see, e.g., \cite{GizewskiMayer21}) that for any function $f\in\HK$ it holds
\begin{eqnarray}\label{err_Sbeta}
\|S_{X_S}^{*}BS_{X_S}f-S_{X_S}^{*}B_{M,N}S_{X_S}f\|_{\HK}\le\kappa^3E_{\beta}\|f\|_{\HK},
\end{eqnarray}
where $E_{\beta}:=\|\beta-\tilde{\beta}\|_{\HK},$ $\tilde{\beta}$ is approximation to $\beta,$ and $\kappa>0$ is some constant. 
\begin{lemma}\label{my:lem5}
For any $0<\delta<1$, with probability at least $1-\delta$, it holds
 \begin{eqnarray}\label{eq_my5}
 \begin{split}
\|(\lambda I&+J_{T}^{*}J_{T})^{\frac{1}{2}}(\lambda I + S_{\X_{S}}^{*}B_{M,N}S_{\X_{S}})^{-\frac{1}{2}}\|_{\HK\rightarrow\HK}\\
&\le C\left (\left(\frac{\mathcal{B}_{m,\lambda}\log\frac{2}{\delta}}{\sqrt{\lambda}}+\left[E_{\beta}+\left(\frac{1}{\sqrt{n}}+\frac{1}{\sqrt{m}}\right)\right]\sqrt{\frac{\mathcal{N}_{\infty}(\lambda)}{\lambda}}\log\frac{2}{\delta}+\frac{1}{2}\right)^{2}+\frac{3}{4}\right)^{\frac{1}{2}}.
\end{split}
\end{eqnarray}
\end{lemma}
The proof of Lemma \ref{my:lem5} is given in Appendix \ref{ap:a}.

For convenience, we denote by $\overline{\mathcal{D}}^{\lambda}_{m,n}$ the quantity
\begin{eqnarray}\nonumber
\begin{split}
\overline{\mathcal{D}}^{\lambda}_{m,n}&=\overline{\mathcal{D}}(\lambda, m,n):=\left(\frac{\mathcal{B}_{m,\lambda}\log\frac{2}{\delta}}{\sqrt{\lambda}}+\left[E_{\beta}+\left(\frac{1}{\sqrt{n}}+\frac{1}{\sqrt{m}}\right)\right]\sqrt{\frac{\mathcal{N}_{\infty}(\lambda)}{\lambda}}\log\frac{2}{\delta}+\frac{1}{2}\right)^{2}+\frac{3}{4}.
\end{split}
\end{eqnarray}

\begin{proposition}\label{result_main}
Assume that in the plain Nystr\"om subsampling the values $|\Znu|$ and $\lambda$ satisfy (\ref{sample_choice}), (\ref{par_choice}), correspondingly. If $f_q$ satisfies Assumption \ref{source_f}, then with probability at least $1-\delta$ it holds 
\begin{eqnarray}\nonumber
\begin{split}
\|f_q-J_Tf_{\z,\Znu}^{\lambda_{M,N}}\|_{L_{2},\rho_{T}}&\le C\varphi(\lambda)
+ C \overline{\mathcal{D}}^{\lambda}_{m,n}
\frac{\mathcal{B}_{m,\lambda}}{\sqrt{\lambda}}
\varphi(\lambda)\log^2\frac{2}{\delta}\\
&+C\overline{\mathcal{D}}^{\lambda}_{m,n}
 \left[E_{\beta} +\left(\frac{1}{\sqrt{n}}+\frac{1}{\sqrt{m}}\right)\right] \sqrt{\frac{\mathcal{N}_{\infty}(\lambda)}{\lambda}}
\varphi(\lambda)\log^2\frac{2}{\delta}\\
&+C \overline{\mathcal{D}}^{\lambda}_{m,n}
\sqrt{\lambda} \left[E_{\beta} + \left(\frac{1}{\sqrt{n}}+\frac{1}{\sqrt{m}}\right)\right] \sqrt{\frac{\mathcal{N}_{\infty}(\lambda)}{\lambda}}
\log^2\frac{2}{\delta}
 \end{split}
\end{eqnarray}
where $f_{\z,\z^{\nu}}^{\lambda_{M,N}}$ is of the form (\ref{app_BMN}). 
\end{proposition}
\begin{proof}
We start with a decomposition
\begin{eqnarray}\label{err_decomp}
\begin{split}
f_{q}-J_{T}f_{\z,\z^{\nu}}^{\lambda_{M,N}}&=f_{q}-
J_{T}\left(\lambda I+P_{\z^{\nu}}S_{X_S}^{\ast} B_{M,N}S_{X_S}P_{\z^{\nu}}\right)^{-1}
P_{\z^{\nu}}S_{X_S}^{\ast} B_{M,N}\overline{y}
=\overline{\sigma}_1+\overline{\sigma}_2+\overline{\sigma}_3,
\end{split}
\end{eqnarray}
where
\begin{eqnarray}
\begin{split}
\overline{\sigma}_1&:=f_{q} -
J_T\left(\lambda I+P_{\z^{\nu}}J_{T}^{\ast}J_{T}P_{\z^{\nu}}\right)^{-1}P_{\z^{\nu}}J_{T}^{\ast}f_q,\\
\overline{\sigma}_2&:=J_T\left(\lambda I+P_{\z^{\nu}}J_{T}^{\ast}J_{T}P_{\z^{\nu}}\right)^{-1}P_{\z^{\nu}}J_{T}^{\ast}f_q- J_{T}\left(\lambda I+P_{\z^{\nu}}S_{X_S}^{\ast} B_{M,N}S_{X_S}P_{\z^{\nu}}\right)^{-1}P_{\z^{\nu}}J_{T}^{\ast}f_q,\nonumber\\
\overline{\sigma}_3&=J_{T}\left(\lambda I+P_{\z^{\nu}}S_{X_S}^{\ast} B_{M,N}S_{X_S}P_{\z^{\nu}}\right)^{-1}P_{\z^{\nu}}J_{T}^{\ast}f_q-
J_{T}\left(\lambda I+P_{\z^{\nu}}S_{X_S}^{\ast} B_{M,N}S_{X_S}P_{\z^{\nu}}\right)^{-1}P_{\z^{\nu}}S_{X_S}^{\ast}B_{M,N}\overline{y}.\\
\nonumber
\end{split}
\end{eqnarray}
 We estimate the norm of each terms
$
\overline{\sigma}_i,$ $i=\overline{1,3}.$ 

Recall (see \cite{VV}), if $g$ is a Borel measurable function on $[0,l^2]$,  $A$ is a linear operator and $\|A\|\le l,$ then it holds
\begin{eqnarray}\label{opp_sim}
\begin{split}
A^{*}g(AA^{*})&=g(A^{*}A)A^{*},\\
Ag(A^{*}A)&=g(AA^{*})A.
\end{split}
\end{eqnarray}
Further, 

\begin{eqnarray}\label{ovsig1_decomp}
\begin{split}
\overline{\sigma}_1&=
\left(I-J_T\left(\lambda I+P_{\z^{\nu}}J_{T}^{\ast}J_{T}P_{\z^{\nu}}\right)^{-1}P_{\z^{\nu}}J_{T}^{\ast}\right)f_q=\left(I-(J_T\PZnu J^{\ast}_{T})\left(\lambda I+J_T\PZnu J^{\ast}_{T}\right)^{-1}\right)\varphi(J_TJ^{\ast}_{T})\mu_q\\
&=\lambda(\lambda I + J_{T}J_{T}^{*})^{-\frac{1}{2}}
(\lambda I + J_{T}J_{T}^{*})^{\frac{1}{2}}(\lambda I +J_{T}\PZnu J_{T}^{*})^{-1}
(\lambda I + J_{T}J_{T}^{*})^{\frac{1}{2}}(\lambda I + J_{T}J_{T}^{*})^{-\frac{1}{2}}\varphi(J_{T}J_{T}^{*})\mu_{q}\nonumber
\end{split}
\end{eqnarray}
and
\begin{eqnarray}\nonumber
\begin{split}
\|\overline{\sigma}_1\|_{L_{2,\rho_{T}}}&\le \lambda\|(\lambda I + J_{T}J_{T}^{*})^{-\frac{1}{2}}\|_{L_{2,\rho_{T}}\rightarrow L_{2,\rho_{T}}}
\|(\lambda I + J_{T}J_{T}^{*})^{\frac{1}{2}}(\lambda I +J_{T}\PZnu J_{T}^{*})^{-\frac{1}{2}}\|_{L_{2,\rho_{T}}\rightarrow L_{2,\rho_{T}} }\\
&\times\|(\lambda I +J_{T}\PZnu J_{T}^{*})^{-\frac{1}{2}}(\lambda I + J_{T}J_{T}^{*})^{\frac{1}{2}}\|_{L_{2,\rho_{T}}\rightarrow L_{2,\rho_{T}} }
\|(\lambda I + J_{T}J_{T}^{*})^{-\frac{1}{2}}\varphi(J_{T}J_{T}^{*})\mu_{q}\|_{L_{2,\rho_{T}} }.
\end{split}
\end{eqnarray}
Meanwhile, note that 
\begin{eqnarray}\nonumber
\begin{split}
\|(\lambda I &+J_{T}\PZnu J_{T}^{*})^{-1}(\lambda I + J_{T}J_{T}^{*})\|_{L_{2,\rho_{T}}\rightarrow L_{2,\rho_{T}} }\le \|(\lambda I +J_{T}\PZnu J_{T}^{*})^{-1}\left(J_{T}J_{T}^{*}-J_{T}\PZnu J_{T}^{*}\right)\|_{L_{2,\rho_{T}}\rightarrow L_{2,\rho_{T}} }+1\\
&\le \|(\lambda I +J_{T}\PZnu J_{T}^{*})^{-1}\|_{L_{2,\rho_{T}}\rightarrow L_{2,\rho_{T}}} 
\|J_T(I-\PZnu)\|^{2}_{\HK\rightarrow L_{2,\rho_{T}}}+1.
\end{split}
\end{eqnarray}
By means of (\ref{eq:g}) and (\ref{eq:prob1}) we get
\begin{eqnarray}\label{b_add_err}
\begin{split}
\|(\lambda I& +J_{T}\PZnu J_{T}^{*})^{-1}(\lambda I + J_{T}J_{T}^{*})\|_{L_{2,\rho_{T}}\rightarrow L_{2,\rho_{T}}} \le \frac{\tilde{\gamma}}{\lambda} \cdot 3\lambda+1=3\tilde{\gamma}+1.
\end{split}
\end{eqnarray}
This together with  Lemma \ref{lem:polar}, Proposition \ref{Cordes}, and (\ref{norm_forL2}), implies that
\begin{eqnarray}\label{sigma1}
\begin{split}
\|\overline{\sigma}_1\|_{L_{2,\rho_{T}}}&\le 
C\phi(\lambda).
\end{split}
\end{eqnarray}

Now, we at the point to estimate the norm of $\overline{\sigma}_2$. Recall that
\begin{eqnarray}\nonumber
\begin{split}
\overline{\sigma}_2&=J_T\left(\left(\lambda I+P_{\z^{\nu}}J_{T}^{\ast}J_{T}P_{\z^{\nu}}\right)^{-1}-\left(\lambda I+P_{\z^{\nu}}S_{X_S}^{\ast} B_{M,N}S_{X_S}P_{\z^{\nu}}\right)^{-1}\right)P_{\z^{\nu}}J_{T}^{\ast}f_q\\
&=J_T\left(\lambda I+P_{\z^{\nu}}S_{X_S}^{\ast} B_{M,N}S_{X_S}P_{\z^{\nu}}\right)^{-1}\PZnu\left [S_{X_S}^{\ast} B_{M,N}S_{X_S}-J_{T}^{\ast}J_{T}\right] \PZnu \left(\lambda I+P_{\z^{\nu}}J_{T}^{\ast}J_{T}P_{\z^{\nu}}\right)^{-1}P_{\z^{\nu}}J_{T}^{\ast}f_q.\\
\end{split}
\end{eqnarray}
From this we have
\begin{eqnarray}\label{sig2:decomp}
\begin{split}
\|\overline{\sigma}_2\|_{L_{2,\rho_{T}}}\le \|\overline{\sigma}_{21}\|_{L_{2,\rho_{T}}}+\|\overline{\sigma}_{22}\|_{L_{2,\rho_{T}}},
\end{split}
\end{eqnarray}
where
\begin{eqnarray}\nonumber
\begin{split}
\overline{\sigma}_{21}&:=J_T\left(\lambda I+P_{\z^{\nu}}S_{X_S}^{\ast} B_{M,N}S_{X_S}P_{\z^{\nu}}\right)^{-1}\PZnu\left [S_{X_S}^{\ast} B_{M,N}S_{X_S}-S_{X_S}^{\ast} BS_{X_S}\right]\left(\lambda I+P_{\z^{\nu}}J_{T}^{\ast}J_{T}P_{\z^{\nu}}\right)^{-1}P_{\z^{\nu}}J_{T}^{\ast}f_q,\\
\overline{\sigma}_{22}&:= J_T\left(\lambda I+P_{\z^{\nu}}S_{X_S}^{\ast} B_{M,N}S_{X_S}P_{\z^{\nu}}\right)^{-1}\PZnu\left [S_{X_S}^{\ast} BS_{X_S}-J_{T}^{\ast}J_{T}\right]\left(\lambda I+P_{\z^{\nu}}J_{T}^{\ast}J_{T}P_{\z^{\nu}}\right)^{-1}P_{\z^{\nu}}J_{T}^{\ast}f_q.
\end{split}
\end{eqnarray}
We are going to estimate each term in the decomposition (\ref{sig2:decomp}). We start with estimation of $\|\overline{\sigma}_{21}\|_{L_{2,\rho_{T}}}.$ 

\begin{eqnarray}\nonumber
\begin{split}
\|\overline{\sigma}_{21}\|_{ L_{2,\rho_{T}}}&\le\|J_{T}\left(\lambda I+J^{\ast}_{T}J_{T}\right)^{-\frac{1}{2}}\|_{\HK\rightarrow L_{2,\rho_{T}}}\|\left(\lambda I+J^{\ast}_{T}J_{T}\right)^{\frac{1}{2}}\left(\lambda I+S_{X_S}^{\ast} B_{M,N}S_{X_S}\right)^{-\frac{1}{2}}\|_{\HK\rightarrow\HK}\\
&\times \|\left(\lambda I+S_{X_S}^{\ast} B_{M,N}S_{X_S}\right)^{\frac{1}{2}}\left(\lambda I+P_{\z^{\nu}}S_{X_S}^{\ast} B_{M,N}S_{X_S}P_{\z^{\nu}}\right)^{-1}\PZnu\left(\lambda I+S_{X_S}^{\ast} B_{M,N}S_{X_S}\right)^{\frac{1}{2}}\|_{\HK\rightarrow\HK}\\
&\times \|\left(\lambda I+S_{X_S}^{\ast} B_{M,N}S_{X_S}\right)^{-\frac{1}{2}}\left(\lambda I+J^{\ast}_{T}J_{T}\right)^{\frac{1}{2}}\|_{\HK\rightarrow\HK}\\
&\times\|\left(\lambda I+J^{\ast}_{T}J_{T}\right)^{-\frac{1}{2}}
\left [S_{X_S}^{\ast} BS_{X_S}-S_{X_S}^{\ast} B_{M,N}S_{X_S}\right]\|_{\HK\rightarrow\HK}\\
&\times \| \left(\lambda I+P_{\z^{\nu}}J_{T}^{\ast}J_{T}P_{\z^{\nu}}\right)^{-1}P_{\z^{\nu}}J_{T}^{\ast}f_q\|_{\HK} .
\end{split}
\end{eqnarray}

Meanwhile, for any $f\in\HK$ it holds
\begin{eqnarray}\label{eq_n5}
 \begin{split}
& \|\left(\lambda I+J^{\ast}_{T}J_{T}\right)^{-\frac{1}{2}}
\left [S_{X_S}^{\ast} BS_{X_S}f-S_{X_S}^{\ast} B_{M,N}S_{X_S}f\right]\|_{\HK} \\
&=\frac{1}{N}\Big\|\sum_{i=1}^{N}\, \left(\lambda I+J^{\ast}_{T}J_{T}\right)^{-\frac{1}{2}} \K(\cdot,x_i)\left(\beta(x_i)-\tilde{\beta}(x_i)\right)f(x_i)\Big\|_{\HK}
\le \kappa^2 \sqrt{\mathcal{N}_{\infty}(\lambda)}\|\beta-\tilde{\beta}\|_{\HK}\|f\|_{\HK}.
 \end{split}
\end{eqnarray}

This together with  (\ref{lem:polar}), (\ref{eq:prob20}), Lemma \ref{my:lem5}, Proposition \ref{Cordes}, and Lemma \ref{my:lem4} imply that 
\begin{eqnarray}\label{ovsig_21}
\begin{split}
\|\overline{\sigma}_{21}\|_{ L_{2,\rho_{T}}}&\le C \left[\left(\frac{\mathcal{B}_{m,\lambda}\log\frac{2}{\delta}}{\sqrt{\lambda}}+\left[E_{\beta}+\left(\frac{1}{\sqrt{n}}+\frac{1}{\sqrt{m}}\right)\right]\sqrt{\frac{\mathcal{N}_{\infty}(\lambda)}{\lambda}}\log\frac{2}{\delta}+\frac{1}{2}\right)^{2}+\frac{3}{4}\right]\\
&\times E_{\beta}  \sqrt{\frac{\mathcal{N}_{\infty}(\lambda)}{\lambda}}
\varphi(\lambda)\log^2\frac{2}{\delta}.
 \end{split}
\end{eqnarray}
Now we are at the point to estimate $\overline{\sigma}_{22}.$ 
\begin{eqnarray}\nonumber
\begin{split}
\|\overline{\sigma}_{22}\|_{ L_{2,\rho_{T}}}&\le\|J_{T}\left(\lambda I+J^{\ast}_{T}J_{T}\right)^{-\frac{1}{2}}\|_{\HK\rightarrow L_{2,\rho_{T}}}\|\left(\lambda I+J^{\ast}_{T}J_{T}\right)^{\frac{1}{2}}\left(\lambda I+S_{X_S}^{\ast} B_{M,N}S_{X_S}\right)^{-\frac{1}{2}}\|_{\HK\rightarrow\HK}\\
&\times \|\left(\lambda I+S_{X_S}^{\ast} B_{M,N}S_{X_S}\right)^{\frac{1}{2}}\left(\lambda I+P_{\z^{\nu}}S_{X_S}^{\ast} B_{M,N}S_{X_S}P_{\z^{\nu}}\right)^{-1}\PZnu\left(\lambda I+S_{X_S}^{\ast} B_{M,N}S_{X_S}\right)^{\frac{1}{2}}\|_{\HK\rightarrow\HK}\\
&\times \|\left(\lambda I+S_{X_S}^{\ast} B_{M,N}S_{X_S}\right)^{-\frac{1}{2}}\left(\lambda I+J^{\ast}_{T}J_{T}\right)^{\frac{1}{2}}\|_{\HK\rightarrow\HK}\\
&\times\|\left(\lambda I+J^{\ast}_{T}J_{T}\right)^{-\frac{1}{2}}
\left [S_{X_S}^{\ast} BS_{X_S}-J^{\ast}_{T}J_{T}\right]\|_{\HK\rightarrow\HK}\\
&\times \| \left(\lambda I+P_{\z^{\nu}}J_{T}^{\ast}J_{T}P_{\z^{\nu}}\right)^{-1}P_{\z^{\nu}}J_{T}^{\ast}f_q\|_{\HK} .
\end{split}
\end{eqnarray}
Applying  (\ref{lem:polar}), Lemma \ref{my:lem5}, (\ref{add_s2}), Lemma \ref{my:lem1}, (\ref{eq:prob20}), Proposition \ref{Cordes}, and Lemma \ref{my:lem4}, with probability at least $1-\delta$ we obtain
\begin{eqnarray}\label{ov_sig22}
\begin{split}
\|\overline{\sigma}_{22}\|_{ L_{2,\rho_{T}}}&\le C
  \left[\left(\frac{\mathcal{B}_{m,\lambda}\log\frac{2}{\delta}}{\sqrt{\lambda}}+\left[E_{\beta}+\left(\frac{1}{\sqrt{n}}+\frac{1}{\sqrt{m}}\right)\right]\sqrt{\frac{\mathcal{N}_{\infty}(\lambda)}{\lambda}}\log\frac{2}{\delta}+\frac{1}{2}\right)^{2}+\frac{3}{4}\right]\\
 &\times \left(\sqrt{\mathcal{N}_{\infty}(\lambda)}\left(\frac{1}{\sqrt{n}}+\frac{1}{\sqrt{m}}\right)+\mathcal{B}_{m,\lambda}\log\frac{2}{\delta}\right)
\frac{1}{\sqrt{\lambda}}\varphi(\lambda).
\end{split}
\end{eqnarray}
Summing up (\ref{ovsig_21}) and (\ref{ov_sig22}), we get
\begin{eqnarray}\label{ov_sig2}
\begin{split}
\|\overline{\sigma}_{2}\|_{ L_{2,\rho_{T}}}& \le
C \left[\left(\frac{\mathcal{B}_{m,\lambda}\log\frac{2}{\delta}}{\sqrt{\lambda}}+\left[E_{\beta}+\left(\frac{1}{\sqrt{n}}+\frac{1}{\sqrt{m}}\right)\right]\sqrt{\frac{\mathcal{N}_{\infty}(\lambda)}{\lambda}}\log\frac{2}{\delta}+\frac{1}{2}\right)^{2}+\frac{3}{4}\right]\\
&\times \left[ E_{\beta}  \sqrt{\frac{\mathcal{N}_{\infty}(\lambda)}{\lambda}}+ \sqrt{\frac{\mathcal{N}_{\infty}(\lambda)}{\lambda}}\left(\frac{1}{\sqrt{n}}+\frac{1}{\sqrt{m}}\right)+\frac{\mathcal{B}_{m,\lambda}}{\sqrt{\lambda}}\log\frac{2}{\delta}\right]
\varphi(\lambda)\log^2\frac{2}{\delta}.
\end{split}
\end{eqnarray}
The rest of the proof is about the estimation of the norm of $\overline{\sigma}_3.$
Recall that
\begin{eqnarray}\nonumber
\begin{split}
\overline{\sigma}_3:=J_{T}\left(\lambda I+P_{\z^{\nu}}S_{X_S}^{\ast} B_{M,N}S_{X_S}P_{\z^{\nu}}\right)^{-1}P_{\z^{\nu}}J_{T}^{\ast}f_q-
J_{T}\left(\lambda I+P_{\z^{\nu}}S_{X_S}^{\ast} B_{M,N}S_{X_S}P_{\z^{\nu}}\right)^{-1}P_{\z^{\nu}}S_{X_S}^{\ast} B_{M,N}\overline{y}.
\end{split}
\end{eqnarray}
As before, we give a decomposition
\begin{eqnarray}\label{ovsig3_dec}
\begin{split}
\|\overline{\sigma}_3\|_{ L_{2,\rho_{T}}}\le \|\overline{\sigma}_{31}\|_{ L_{2,\rho_{T}}}+ \|\overline{\sigma}_{32}\|_{ L_{2,\rho_{T}}},
\end{split}
\end{eqnarray}
where
\begin{eqnarray}\nonumber
\begin{split}
\overline{\sigma}_{31}&:= J_{T}\left(\lambda I+P_{\z^{\nu}}S_{X_S}^{\ast} B_{M,N}S_{X_S}P_{\z^{\nu}}\right)^{-1}P_{\z^{\nu}}\left(J_{T}^{\ast}f_q-
S_{X_S}^{\ast} B\overline{y}\right),\\
\overline{\sigma}_{32}&:= J_{T}\left(\lambda I+P_{\z^{\nu}}S_{X_S}^{\ast} B_{M,N}S_{X_S}P_{\z^{\nu}}\right)^{-1}P_{\z^{\nu}}\left(S_{X_S}^{\ast} B\overline{y}-S_{X_S}^{\ast} B_{M,N}\overline{y}\right).
\end{split}
\end{eqnarray}
We are going to estimate each of the norm in the right-hand side of (\ref{ovsig3_dec}).
\begin{eqnarray}\nonumber
\begin{split}
\|\overline{\sigma}_{31}\|_{ L_{2,\rho_{T}}}&\le\|J_{T}\left(\lambda I+J^{\ast}_{T}J_{T}\right)^{-\frac{1}{2}}\|_{\HK\rightarrow L_{2,\rho_{T}}}\|\left(\lambda I+J^{\ast}_{T}J_{T}\right)^{\frac{1}{2}}\left(\lambda I+S_{X_S}^{\ast} B_{M,N}S_{X_S}\right)^{-\frac{1}{2}}\|_{\HK\rightarrow\HK}\\
&\times \|\left(\lambda I+S_{X_S}^{\ast} B_{M,N}S_{X_S}\right)^{\frac{1}{2}}\left(\lambda I+P_{\z^{\nu}}S_{X_S}^{\ast} B_{M,N}S_{X_S}P_{\z^{\nu}}\right)^{-1}\PZnu\left(\lambda I+S_{X_S}^{\ast} B_{M,N}S_{X_S}\right)^{\frac{1}{2}}\|_{\HK\rightarrow\HK}\\
&\times \|\left(\lambda I+S_{X_S}^{\ast} B_{M,N}S_{X_S}\right)^{-\frac{1}{2}}\left(\lambda I+J^{\ast}_{T}J_{T}\right)^{\frac{1}{2}}\|_{\HK\rightarrow\HK}\\
&\times\|\left(\lambda I+J^{\ast}_{T}J_{T}\right)^{-\frac{1}{2}}
\left(J_{T}^{\ast}f_q-
S_{X_S}^{\ast} B\overline{y}\right)]\|_{\HK\rightarrow\HK}
\end{split}
\end{eqnarray}
By means of  (\ref{lem:polar}), (\ref{bound:2}), Proposition \ref{Cordes},  Lemma \ref{my:lem5}, and Lemma \ref{my:lem3}, we obtain
\begin{eqnarray}\label{ovsig_31}
\begin{split}
\|\overline{\sigma}_{31}\|_{ L_{2,\rho_{T}}}&\le C  \left[\left(\frac{\mathcal{B}_{m,\lambda}\log\frac{2}{\delta}}{\sqrt{\lambda}}+\left[E_{\beta}+\left(\frac{1}{\sqrt{n}}+\frac{1}{\sqrt{m}}\right)\right]\sqrt{\frac{\mathcal{N}_{\infty}(\lambda)}{\lambda}}\log\frac{2}{\delta}+\frac{1}{2}\right)^{2}+\frac{3}{4}\right]\\
&\times  \sqrt{\frac{\mathcal{N}_{\infty}(\lambda)}{n}}
\log^2\frac{2}{\delta}.
 \end{split}
\end{eqnarray}
Further, we estimate $\overline{\sigma}_{32}.$
\begin{eqnarray}\nonumber
\begin{split}
\|\overline{\sigma}_{32}\|_{ L_{2,\rho_{T}}}&\le\|J_{T}\left(\lambda I+J^{\ast}_{T}J_{T}\right)^{-\frac{1}{2}}\|_{\HK\rightarrow L_{2,\rho_{T}}}\|\left(\lambda I+J^{\ast}_{T}J_{T}\right)^{\frac{1}{2}}\left(\lambda I+S_{X_S}^{\ast} BS_{X_S}\right)^{-\frac{1}{2}}\|_{\HK\rightarrow\HK}\\
&\times\|\left(\lambda I+S_{X_S}^{\ast} BS_{X_S}\right)^{\frac{1}{2}}\left(\lambda I+S_{X_S}^{\ast} B_{M,N}S_{X_S}\right)^{-\frac{1}{2}}\|_{\HK\rightarrow\HK}\\
&\times \|\left(\lambda I+S_{X_S}^{\ast} B_{M,N}S_{X_S}\right)^{\frac{1}{2}}\left(\lambda I+P_{\z^{\nu}}S_{X_S}^{\ast} B_{M,N}S_{X_S}P_{\z^{\nu}}\right)^{-1}\PZnu\left(\lambda I+S_{X_S}^{\ast} B_{M,N}S_{X_S}\right)^{\frac{1}{2}}\|_{\HK\rightarrow\HK}\\
&\times \|\left(\lambda I+S_{X_S}^{\ast} B_{M,N}S_{X_S}\right)^{-\frac{1}{2}}\left(\lambda I+J^{\ast}_{T}J_{T}\right)^{\frac{1}{2}}\|_{\HK\rightarrow\HK}\\
&\times\|\left(\lambda I+J^{\ast}_{T}J_{T}\right)^{-\frac{1}{2}}
\left(S_{X_S}^{\ast} B\overline{y}-S_{X_S}^{\ast} B_{M,N}\overline{y}\right)]\|_{\HK\rightarrow\HK}
\end{split}
\end{eqnarray}
Note that for any $\overline{y}\in\mathbb{R}^{n}$ we have
\begin{eqnarray}\nonumber
 \begin{split}
& \|\left(\lambda I+J^{\ast}_{T}J_{T}\right)^{-\frac{1}{2}}
\left [S_{X_S}^{\ast} B\overline{y}-S_{X_S}^{\ast} B_{M,N}\overline{y}\right]\|_{\HK} 
=\frac{1}{N}\Big\|\sum_{i=1}^{N}\, \left(\lambda I+J^{\ast}_{T}J_{T}\right)^{-\frac{1}{2}} \K(\cdot,x_i)\left(\beta(x_i)-\tilde{\beta}(x_i)\right)y_i\Big\|_{\HK}\\
&\le \kappa^2 y_0\sqrt{\mathcal{N}_{\infty}(\lambda)}\|\beta-\tilde{\beta}_{M,N,\Znu}^{\alpha_{M,N}}\|_{\HK}\le C E_{\beta}\sqrt{\mathcal{N}_{\infty}(\lambda)}.
 \end{split}
\end{eqnarray}
This together with Lemma \ref{lem:polar}, (\ref{eq:prob2}), (\ref{bound:2}), Lemma \ref{my:lem5}, Proposition \ref{Cordes}, and Lemma  \ref{my:lem1} implies that
\begin{eqnarray}\label{ovsig_32}
\begin{split}
\|\overline{\sigma}_{32}\|_{ L_{2,\rho_{T}}}&\le C  \left[\left(\frac{\mathcal{B}_{m,\lambda}\log\frac{2}{\delta}}{\sqrt{\lambda}}+\left[E_{\beta}+\left(\frac{1}{\sqrt{n}}+\frac{1}{\sqrt{m}}\right)\right]\sqrt{\frac{\mathcal{N}_{\infty}(\lambda)}{\lambda}}\log\frac{2}{\delta}+\frac{1}{2}\right)^{2}+\frac{3}{4}\right]\\
&\times E_{\beta}  \sqrt{\mathcal{N}_{\infty}(\lambda)}
\log\frac{2}{\delta}.
 \end{split}
\end{eqnarray}
Summing up (\ref{ovsig_31}) and (\ref{ovsig_32}), we derive
\begin{eqnarray}\label{ovsig_3}
\begin{split}
\|\overline{\sigma}_{3}\|_{ L_{2,\rho_{T}}}&\le C \left[\left(\frac{\mathcal{B}_{m,\lambda}\log\frac{2}{\delta}}{\sqrt{\lambda}}+\left[E_{\beta}+\left(\frac{1}{\sqrt{n}}+\frac{1}{\sqrt{m}}\right)\right]\sqrt{\frac{\mathcal{N}_{\infty}(\lambda)}{\lambda}}\log\frac{2}{\delta}+\frac{1}{2}\right)^{2}+\frac{3}{4}\right] \\
&\times \left[ \sqrt{\frac{\mathcal{N}_{\infty}(\lambda)}{n}} +E_{\beta}  \sqrt{\mathcal{N}_{\infty}(\lambda)}\right]
\log^2\frac{2}{\delta}.
 \end{split}
\end{eqnarray}
Combining (\ref{sigma1}), (\ref{ov_sig2}) and (\ref{ovsig_3}), with probability at least $1-\delta$ we finally obtain
\begin{eqnarray}\nonumber
\begin{split}
\|f_q-J_Tf_{\z,\Znu}^{\lambda_{M,N}}\|_{L_{2},\rho_{T}}&\le C\varphi(\lambda)
+ C \overline{\mathcal{D}}^{\lambda}_{m,n}
\frac{\mathcal{B}_{n,\lambda}}{\sqrt{\lambda}}
\varphi(\lambda)\log^2\frac{2}{\delta}\\
&+C\overline{\mathcal{D}}^{\lambda}_{m,n}
 \left[E_{\beta} +\left(\frac{1}{\sqrt{n}}+\frac{1}{\sqrt{m}}\right)\right] \sqrt{\frac{\mathcal{N}_{\infty}(\lambda)}{\lambda}}
\varphi(\lambda)\log^2\frac{2}{\delta}\\
&+C \overline{\mathcal{D}}^{\lambda}_{m,n}
\sqrt{\lambda} \left[E_{\beta} + \left(\frac{1}{\sqrt{n}}+\frac{1}{\sqrt{m}}\right)\right] \sqrt{\frac{\mathcal{N}_{\infty}(\lambda)}{\lambda}}
\log^2\frac{2}{\delta}.
 \end{split}
\end{eqnarray}
The proposition  is completely proved.
\end{proof}
\subsection{Refinement of main result}
Recall, $E_{\beta}:=\|\beta-\tilde{\beta}\|_{\HK}$ and
assume that  $\beta=\frac{d\rho_T}{d\rho_S}$ satisfies Assumption 
\ref{source_beta}.

The following result will be used in the sequel.
\begin{proposition}\cite[Corollary 4.4]{SolMylRecov25}\label{col_HK}
Denote $\theta_{\psi,\zeta}(t)=\frac{t\psi(t)}{\zeta(t)}$ and $\alpha=\alpha_{N,M}:=\theta_{\psi,\zeta}^{-1}\left(\frac{1}{\sqrt{N}}+\frac{1}{\sqrt{M}}\right),$ then with probability at least $1-\delta$ it holds
\begin{equation}\label{err_HK1}
E_{\beta}\le C\psi\left(\theta_{\psi,\zeta}^{-1}\left(\frac{1}{\sqrt{N}}+\frac{1}{\sqrt{M}}\right)\right)\log^2\frac{2}{\delta}.
\end{equation}
\end{proposition}

By means of Proposition \ref{col_HK}, we reformulate the estimate from Proposition \ref{result_main} as follows
 \begin{eqnarray}\nonumber
\begin{split}
&\|f_q-J_Tf_{\z,\Znu}^{\lambda_{M,N}}\|_{L_{2},\rho_{T}}\le C\varphi(\lambda)\\
&+C\overline{\mathcal{D}}^{\lambda}_{m,n}
\left[\psi\left(\theta_{\psi,\zeta}^{-1}\left(\frac{1}{\sqrt{N}}+\frac{1}{\sqrt{M}}\right)\right) +\left(\frac{1}{\sqrt{n}}+\frac{1}{\sqrt{m}}\right)\right]  \sqrt{\frac{\mathcal{N}_{\infty}(\lambda)}{\lambda}}
\varphi(\lambda)\log^2\frac{2}{\delta}\\
&\le C
 \left(1+  \overline{\mathcal{D}}^{\lambda}_{m,n}\left[\psi\left(\theta_{\psi,\zeta}^{-1}\left(\frac{1}{\sqrt{N}}+\frac{1}{\sqrt{M}}\right)\right) +\left(\frac{1}{\sqrt{n}}+\frac{1}{\sqrt{m}}\right)\right]  \sqrt{\frac{\mathcal{N}_{\infty}(\lambda)}{\lambda}}\right) \varphi(\lambda)\log^2\frac{2}{\delta}.
 \end{split}
\end{eqnarray}
Applying Lemma \ref{lem:dem} we get
\begin{eqnarray}\nonumber
\begin{split}
\|f_q-J_Tf_{\z,\Znu}^{\lambda_{M,N}}\|_{L_{2},\rho_{T}}&\le \overline{C}\left(1+  \left[\psi\left(\theta_{\psi,\zeta}^{-1}\left(\frac{1}{\sqrt{N}}+\frac{1}{\sqrt{M}}\right)\right) +\left(\frac{1}{\sqrt{n}}+\frac{1}{\sqrt{m}}\right)\right]\frac{ \zeta(\lambda)}{\lambda}\right) \varphi(\lambda)\log^2\frac{2}{\delta}.
 \end{split}
\end{eqnarray}
Thus, we proof the following statement.
\begin{theorem}\label{main}
Assume that in the plain Nystr\"om subsampling the values $|\Znu|$ and $\lambda$ satisfy (\ref{sample_choice}), (\ref{par_choice}), correspondingly. If $f_q$ satisfies Assumption \ref{source_f}, then with probability at least $1-\delta$ it holds 
\begin{eqnarray}\nonumber
\begin{split}
\|f_q-J_Tf_{\z,\Znu}^{\lambda_{M,N}}\|_{L_{2},\rho_{T}}&\le \overline{C}\left(1+  \left[\psi\left(\theta_{\psi,\zeta}^{-1}\left(\frac{1}{\sqrt{N}}+\frac{1}{\sqrt{M}}\right)\right) +\left(\frac{1}{\sqrt{n}}+\frac{1}{\sqrt{m}}\right)\right]\frac{ \zeta(\lambda)}{\lambda}\right) \varphi(\lambda)\log^2\frac{2}{\delta}.
 \end{split}
\end{eqnarray}
\end{theorem}

According to  Theorem \ref{main}, the parameter $\lambda$ should be selected as
$$\left[\psi\left(\theta_{\psi,\zeta}^{-1}\left(\frac{1}{\sqrt{N}}+\frac{1}{\sqrt{M}}\right)\right) +\left(\frac{1}{\sqrt{n}}+\frac{1}{\sqrt{m}}\right)\right]\frac{ \zeta(\lambda)}{\lambda}=1$$ or
\begin{equation}\label{xi_eq2}
\frac{\lambda}{ \zeta(\lambda)}=\psi\left(\theta_{\psi,\zeta}^{-1}\left(\frac{1}{\sqrt{N}}+\frac{1}{\sqrt{M}}\right)\right) +\left(\frac{1}{\sqrt{n}}+\frac{1}{\sqrt{m}}\right).
\end{equation}

Thus, it leads to the following statement.
\begin{corollary}
 If $f_q$ meets Assumption \ref{source_f} with $\varphi(t)=t^{\gamma},$ $\beta$ meets Assumption \ref{source_cond_beta} with $\psi(t)=t^{\eta},$ and $\K$ meets Assumption \ref{source_ker} with $\zeta(t)=t^{r},$ where $\gamma, \eta, r \in (0,\frac{1}{2}],$ then
 \begin{eqnarray}\label{pover_func2}
\begin{split}
\|f_q-J_Tf_{\z,\Znu}^{\lambda_{M,N}}\|_{L_{2},\rho_{T}}&\le \overline{C}\left[\left(\frac{1}{\sqrt{N}}+\frac{1}{\sqrt{M}}\right)^{\frac{\eta}{1+\eta-r}}+\left(\frac{1}{\sqrt{n}}+\frac{1}{\sqrt{m}}\right)\right]^{\frac{\gamma}{1-r}}\log^2\frac{2}{\delta}.
 \end{split}
\end{eqnarray}
\end{corollary}
\begin{proof}
Due to Proposition \ref{col_HK} we have
$$
\alpha^{1+\eta-r}_{M,N}=\frac{1}{\sqrt{N}}+\frac{1}{\sqrt{M}} \quad
\Rightarrow \qquad \alpha_{M,N}=\left(\frac{1}{\sqrt{N}}+\frac{1}{\sqrt{M}}\right)^{\frac{1}{1+\eta-r}}.
$$
From (\ref{xi_eq2}) it follows
$$
\lambda^{r-1}=\left[\left(\frac{1}{\sqrt{N}}+\frac{1}{\sqrt{M}}\right)^{\frac{\eta}{1+\eta-r}}+\left(\frac{1}{\sqrt{n}}+\frac{1}{\sqrt{m}}\right)\right]^{-1}  $$ or
$$  \lambda = \left[\left(\frac{1}{\sqrt{N}}+\frac{1}{\sqrt{M}}\right)^{\frac{\eta}{1+\eta-r}}+\left(\frac{1}{\sqrt{n}}+\frac{1}{\sqrt{m}}\right)\right]^{\frac{1}{1-r}}.
$$
Hence, applying Theorem \ref{result_main} we get
\begin{eqnarray}\nonumber
\begin{split}
\|f_q-J_Tf_{\z,\Znu}^{\lambda_{M,N}}\|_{L_{2},\rho_{T}}&\le \overline{C}\left[\left(\frac{1}{\sqrt{N}}+\frac{1}{\sqrt{M}}\right)^{\frac{\eta}{1+\eta-r}}+\left(\frac{1}{\sqrt{n}}+\frac{1}{\sqrt{m}}\right)\right]^{\frac{\gamma}{1-r}}\log^2\frac{2}{\delta}.
 \end{split}
\end{eqnarray}
\end{proof}
\begin{remark}
Let's compare accuracy estimates of approximation $f_q$ in the case when the Radon-Nikodym derivative $\beta$ is given (see (\ref{pover_func1})) and in the case of the unknown  Radon-Nikodym derivative  (\ref{pover_func2}). It is easy to see, that in  (\ref{pover_func2}) the extra term
\begin{equation}\label{add_term}
\left(\frac{1}{\sqrt{N}}+\frac{1}{\sqrt{M}}\right)^\frac{\eta}{1+\eta-r}
\end{equation}
 is occurred.
This is due to necessity of the additional discretization parameters  $M$ and $N$ in the case of the unknown  Radon-Nikodym derivative. If $M$ and $N$ are chosen  such that $\frac{1}{\sqrt{N}}+\frac{1}{\sqrt{M}}=O\left((\frac{1}{\sqrt{n}}+\frac{1}{\sqrt{m}})^{\frac{1+\eta-r}{\eta}}\right)$ then it allows us to guarantee the same order of accuracy as in (\ref{pover_func1}). Moreover, such choice ensure the minimum sample size $M$ and $N$ that preserve the convergence rate from Corollary \ref{col_1}, 
\end{remark}

\section{Numerical illustration}
In this section we provide  numerical experiments that support our theoretical results. The experiments were conducted in MATLAB on a system equipped with an Intel® Xeon® Platinum 8280 CPU @ 2.70 GHz, 128 GB of RAM, running Windows Server 2022, without the use of a GPU.\\ \\

Let $s\in(0,1).$ In view of \cite{Slob}, we define  the fractional  Sobolev space $W_2^{s}(0,1)$ as follows
 \begin{eqnarray}\label{sobol_space_gen}
\mathcal{W}_2^{s}(0,1):=\left\{u\in L_2(0,1)\colon \int_{0}^{1}\int_{0}^{1} \frac{|u(x)-u(y)|^{2}}{|x-y|^{n+2s}}dxdy<\infty\right\}
\end{eqnarray}
which endowed with the norm 
\begin{equation}\label{sobol_norm_gen}
\|u\|^{2}_{\mathcal{W}_2^{s}(0,1)}=\|u\|^{2}_{L_2(0,1)}+[u]^{2}_{\mathcal{W}_2^{s}(0,1)}=\int_{0}^{1} u^{2}(x)dx+
 \int_{0}^{1}\int_{0}^{1} \frac{|u(x)-u(y)|^{2}}{|x-y|^{n+2s}}dxdy.
\end{equation}
It is worth noting that  fractional Sobolev spaces are also referred to as Aronszajn, Gagliardo, or Sobolev-Slobodeckij spaces (see \cite{Slob}, \cite{Aron}, \cite{Gag}).

We choose the kernel as 
\begin{equation}\label{fBm_V}
\K(x,y)=1+\frac{1}{2}\left(x^{\frac{1}{2}}+y^{\frac{1}{2}}-|x-y|^{\frac{1}{2}}\right), \quad x,y\in[0,1].
\end{equation} 
It is known (see, e.g., \cite{MolchanGolosov69}, \cite{AlpLevan08}, \cite{Picard11}) that the kernel (\ref{fBm_V}) generates RKHS $\HK=\mathcal{W}_{2}^{\frac{3}{4}}(0,1),$ which is the fractional Sobolev space (\ref{sobol_space_gen}) with $s=\frac{3}{4}$.
In what follows we will use the fact that any function from the Sobolev space $\mathcal{W}_{2}^{s}(0,1)$ is continuous whenever $s > \frac{1}{2}$ (see, e.g., \cite{AdamsFournier}). 
More precise, to treat our condition $f\notin \HK,$ we are going to approximate discontinuous functions by continuous functions generated by the kernel (\ref{fBm_V}). Namely, we take
\begin{equation}\label{exact_func}
f(x)=\left\{ \begin{array}{cl}
 (x-0.5)^2  +0.5(x-0.5)+1, &  0\le x\le 0.5, \\ \\
(x-0.5)^4 -4(x-0.5)^2+2(x-0.5)+ 0.75, & \quad 0.5\le x\le1,
\end{array} \right.
\end{equation}
which has a discontinuity in $x=0.5$

Further, we simulate input-output pairs. Let's inputs $\X_{T}=\{x_1^{'},x_2^{'},\ldots, x_m^{'}\}$ in the target domain are sampled from the continuous uniform distribution $\rho_T\sim U(0,1)$ over $[0,1],$ and   
in the source domain $\X_{S}=\{x_1,x_2,\ldots, x_n\}$ are sampled from the beta distribution $\rho_S\sim B(\frac{1}{2},1)$. In this case (see, e.g. \cite{NgPer23}), the Radon-Nikodym derivative is known to be
$$
\beta(x)=2\sqrt{x}, \quad x\in[0,1].
$$
The outputs are simulated as follows:
$$
y_{i}=f(x_i)+\epsilon_i,\quad i=1,2,\ldots,n,
$$
where $\epsilon_i$ are zero-mean Gaussian random variables with standard deviation $\delta,$ and $f$ is defined by (\ref{exact_func}).

The numerical tests were conducted with $N = n = m = M = \{10^{4}, 10^{5}, 10^{6}\}$. In view of Theorem~\ref{main}, the regularization parameter $\lambda$ is of order $O(N^{-1}).$ Correspondingly, the optimal theoretical expected error is of order $O(10^{-2})$ for $N = 10^{4}$ and $O(10^{-3})$ for $N = 10^{5}$ and $N=10^{6}$.
The summary of the performance is presented in tables and plots below.  \\ \\
As can be seen from the plots, the main error is concentrated near the point of discontinuity.
\begin{figure}[ht]
\centering
\begin{minipage}{0.32\textwidth}
\centering
\includegraphics[width=\linewidth]{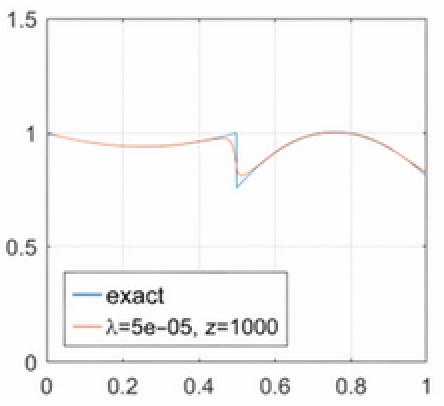}
(a)
\end{minipage}
\hfill
\begin{minipage}{0.32\textwidth}
\centering
\includegraphics[width=\linewidth]{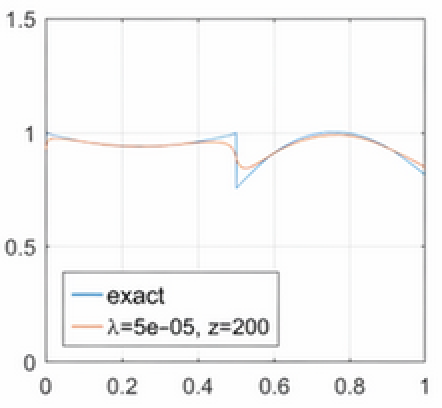}
(b) 
\end{minipage}
\hfill
\begin{minipage}{0.32\textwidth}
\centering
\includegraphics[width=\linewidth]{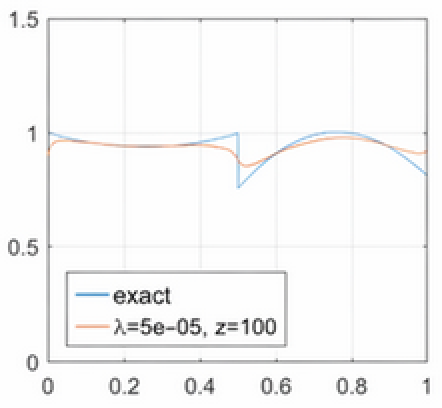}
(c) 
\end{minipage}
\end{figure}

The plots show simulations for $N=10^{5}, \, \lambda= 5\times 10^{-5},$ and  $\Znu=1000$ (plot (a)), $ \Znu=200$  (plot (b)), and $\Znu=100$  (plot (c)), correspondingly.

\begin{table}[H]
\label{tab:znu1}
\centering
\Large
\begin{tabular}{|c|c|c|c|c|c|}

\hline
 & $N$ & \multicolumn{4}{c|}{$\Znu$} \\
\cline{2-6}
$\lambda$ & $10^4$ & $10^3$ & $5\times10^2$ & $2\times10^2$ & $10^2$ \\
\hline
$10^{-4}$          & 0.007 & 0.017 & 0.021 & 0.029 & 0.041 \\
\hline
$5\times10^{-5}$   & 0.006 & 0.014 & 0.017 & 0.022 & 0.033 \\
\hline
$2.5\times10^{-5}$ & 0.005 & 0.012 & 0.014 & 0.018 & 0.027 \\
\hline
\end{tabular}
\caption*{Table 1}
\end{table}

\begin{table}[H]
\label{tab:znu2}
\centering
\Large
\begin{tabular}{|c|c|c|c|c|c|}
\hline
 & $N$ & \multicolumn{4}{c|}{$\Znu$} \\
\cline{2-6}
$\lambda$ & $10^5$ & $10^4$ & $5 \times 10^3$ & $2 \times 10^3$ & $10^3$ \\
\hline
$10^{-5}$ & 0.003 & 0.009 & 0.011 & 0.015 & 0.018 \\
\hline
$5 \times 10^{-6}$ & 0.003 & 0.007 & 0.009 & 0.012 & 0.015 \\
\hline
$2.5 \times 10^{-6}$ & 0.002 & 0.006 & 0.007 & 0.009 & 0.012 \\
\hline
\end{tabular}
\caption*{Table 2}
\end{table}

\begin{table}[H]
\label{tab:znu3}
\centering
\Large
\begin{tabular}{|c|c|c|c|c|c|}
\hline
 & $N$ & \multicolumn{4}{c|}{$\Znu$} \\
\cline{2-6}
$\lambda$ & $10^6$ & $10^5$ & $5 \times 10^4$ & $2 \times 10^4$ & $10^4$ \\
\hline
$10^{-6}$ &  & 0.003 & 0.004 & 0.006 & 0.008 \\
\hline
$5 \times 10^{-7}$ &  & 0.003 & 0.003 & 0.005 & 0.006 \\
\hline
$2.5 \times 10^{-7}$ &  & 0.002 & 0.003 & 0.004 & 0.005 \\
\hline
\end{tabular}
\caption*{Table 3}
\end{table}

In Table 3 the computations for the first column ($N=10^6$) were not performed due to computational resource limitations caused by the large amount of input data.

From the results presented in Tables 1–3, it follows that the expected order of accuracy is preserved even when the amount of available information is reduced by a factor of 10 to 100.

\section{Conclusions}
We have presented a convergence analysis of the regularized Nystr\"om subsampling method applied to unsupervised domain adaptation under covariate shift in the low-smoothness settings. The framework covers the misspecified case in which the target function lies outside the reproducing kernel Hilbert space.

When the Radon–Nikodym derivative $\beta$ is known, Corollary \ref{col_1} provides explicit convergence rates on the power scale: the excess risk is bounded by $C(n^{-1/2} + m^{-1/2})^{\gamma/(1-r)}$ up to logarithmic factors, where $\gamma$ and $r$ are the smoothness parameters of the target function and the kernel sections, respectively.

We extended the analysis to the practically relevant setting where $\beta$ is unknown and must be estimated from additional samples of sizes $N$ and $M$. Theorem \ref{main} shows that the convergence rate depends on an additional parameters $N$ and $M$ due to the error estimate  of the Radon-Nikodym derivative $\beta$. 
Theoretical results  are supported by numerical experiments.

\appendix
\newpage

\renewcommand{\thetheorem}{A.\arabic{theorem}}
\renewcommand{\theproposition}{A.\arabic{proposition}}
\renewcommand{\thedefinition}{A.\arabic{definition}}
\renewcommand{\thecorollary}{A.\arabic{corollary}}
\renewcommand{\thelemma}{A.\arabic{lemma}}
\renewcommand{\theremark}{A.\arabic{remark}}
\renewcommand{\theexample}{A.\arabic{example}}
\renewcommand{\theequation}{A.\arabic{equation}}
\section{Appendix. Proof of Auxiliary Statements}\label{ap:a}
{\bf{Proof of Lemma \ref{my:lem2}.}} \\
To prove the inequality (\ref{eq_my2}), we consider two random variables defined as
\begin{eqnarray}\nonumber
\begin{split}
\xi_1(x)&:=\left(\lambda I+J^{\ast}_{T}J_{T}\right)^{-\frac{1}{2}}\K(\cdot,x)f(x),\\
\xi_2(x)&:=\left(\lambda I+J^{\ast}_{T}J_{T}\right)^{-\frac{1}{2}}\K(\cdot,x)\beta(x)f(x).
\end{split}
\end{eqnarray}
Further,
$$\|\xi_1(x)\|_{\HK}=\|\left(\lambda I+J^{\ast}_{T}J_{T}\right)^{-\frac{1}{2}}\K(\cdot,x)f(x)\|_{\HK}\le \sqrt{\mathcal{N}_{\infty}(\lambda)}\|f\|_{\HK},$$
$$\|\xi_2(x)\|_{\HK}=\|\left(\lambda I+J^{\ast}_{T}J_{T}\right)^{-\frac{1}{2}}\K(\cdot,x)\beta(x)f(x)\|_{\HK}\le \beta_0\sqrt{\mathcal{N}_{\infty}(\lambda)}\|f\|_{\HK}.$$
Moreover,
\begin{eqnarray}\nonumber
\begin{split}
\mathbb{E}\xi_1&=\int_{X} 
\left(\lambda I+J^{\ast}_{T}J_{T}\right)^{-\frac{1}{2}}\K(\cdot,x)f(x)d\rho_T(x)\\
&=
\int_{X} 
\left(\lambda I+J^{\ast}_{T}J_{T}\right)^{-\frac{1}{2}}\K(\cdot,x)\beta(x)f(x)d\rho_{S}(x)
=\mathbb{E}\xi_2.
\end{split}
\end{eqnarray}

From this, we have that $\mathbb{E}\xi_1=\mathbb{E}\xi_2$.

In addition, 
\begin{eqnarray}\nonumber
\mathbb{E}\|\xi_1\|^{2}_{\HK}=\int_{X} 
\|\left(\lambda I+J^{\ast}_{T}J_{T}\right)^{-\frac{1}{2}}\K(\cdot,x)f(x)\|^2_{\HK}d\rho_{T}(x)\le\left(C\sqrt{\mathcal{N}_{\infty}(\lambda)}\|f(x)\|_{\HK}\right)^{2}.
\end{eqnarray}
In the same way we estimate
\begin{eqnarray}\nonumber
\begin{split}
\mathbb{E}\|\xi_2\|^{2}_{\HK}&=\int_{X} 
\|\left(\lambda I+J^{\ast}_{T}J_{T}\right)^{-\frac{1}{2}}\K(\cdot,x)\beta(x)f(x)\|^2_{\HK}d\rho_{S}(x)\\
&=\int_{X} 
\|\left(\lambda I+J^{\ast}_{T}J_{T}\right)^{-\frac{1}{2}}\K(\cdot,x)\|^{2}_{\HK}\|\beta(x)\|^{2}_{\HK}\|f(x)\|^{2}_{\HK}d\rho_S(x)
\le \left(C\sqrt{\mathcal{N}_{\infty}(\lambda)}\|f(x)\|_{\HK}\right)^{2}.
\end{split}
\end{eqnarray}
Applying Proposition  \ref{concentrat}, with probability at least $1-\delta$ we obtain 
\begin{eqnarray}\nonumber
\begin{split}
&\|(\lambda I+J^{\ast}_{T}J_{T})^{-\frac{1}{2}}(S_{X_S}^{\ast} BS_{X_S}-S_{X_{T}}^{\ast}S_{X_{T}})\|_{\HK\rightarrow\HK}
\le \|\frac{1}{m}\sum_{j=1}^{m}\,\xi_1(x^{'}_j)-\mathbb{E}\xi_1\|_{\HK}+\|\frac{1}{n}\sum_{i=1}^{n}\,\xi_2(x_i)-\mathbb{E}\xi_2\|_{\HK}\\
&\le C\sqrt{\mathcal{N}_{\infty}(\lambda)}\left(\frac{1}{\sqrt{n}}+\frac{1}{\sqrt{m}}\right)\log\frac{2}{\delta}.
\end{split}
\end{eqnarray}
Lemma is proved.
\\
\\
{\bf{Proof of Lemma \ref{my:lem1}.}} \\
To prove the inequality (\ref{eq_my1}), first, note that $AB^{-1}=(A-B)A^{-1}(A-B)B^{-1}+(A-B)A^{-1}+I,$ then
\begin{eqnarray}\label{dec:1}
 \begin{split}
(\lambda I&+J_{T}^{*}J_{T})(\lambda I + S_{\X_{S}}^{*}BS_{\X_{S}})^{-1}\\
&=(J_{T}^{*}J_{T}-S_{\X_{S}}^{*}BS_{\X_{S}})(\lambda I+J_{T}^{*}J_{T})^{-1}(J_{T}^{*}J_{T}-S_{\X_{S}}^{*}BS_{\X_{S}})(\lambda I + S_{\X_{S}}^{*}BS_{\X_{S}})^{-1}\\
&+(J_{T}^{*}J_{T}-S_{\X_{S}}^{*}BS_{\X_{S}})(\lambda I + J_{T}^{*}J_{T})^{-1}+I.
\end{split}
\end{eqnarray}
Now, let's decompose the first term in (\ref{dec:1}):
\begin{eqnarray}\nonumber 
 \begin{split}
&(J_{T}^{*}J_{T}-S_{\X_{S}}^{*}BS_{\X_{S}})(\lambda I+J_{T}^{*}J_{T})^{-1}(J_{T}^{*}J_{T}-S_{\X_{S}}^{*}BS_{\X_{S}})(\lambda I + S_{\X_{S}}^{*}BS_{\X_{S}})^{-1}\\
&=\left[(J_{T}^{*}J_{T}-S_{\X_{T}}^{*}S_{\X_{T}})(\lambda I+J_{T}^{*}J_{T})^{-\frac{1}{2}}+(S_{\X_{T}}^{*}S_{\X_{T}}-S_{\X_{S}}^{*}BS_{\X_{S}})(\lambda I+J_{T}^{*}J_{T})^{-\frac{1}{2}}\right]\\
&\times\left[(\lambda I+J_{T}^{*}J_{T})^{-\frac{1}{2}}(J_{T}^{*}J_{T}-S_{\X_{T}}^{*}S_{\X_{T}})+(\lambda I+J_{T}^{*}J_{T})^{-\frac{1}{2}}(S_{\X_{T}}^{*}S_{\X_{T}}-S_{\X_{S}}^{*}BS_{\X_{S}})\right](\lambda I + S_{\X_{S}}^{*}BS_{\X_{S}})^{-1}.
\end{split}
\end{eqnarray}
For the second term  in (\ref{dec:1}) we have
\begin{eqnarray}\nonumber
 \begin{split}
 &(J_{T}^{*}J_{T}-S_{\X_{S}}^{*}BS_{\X_{S}})(\lambda I + J_{T}^{*}J_{T})^{-1}\\
&=\left[(J_{T}^{*}J_{T}-S_{\X_{T}}^{*}S_{\X_{T}})(\lambda I+J_{T}^{*}J_{T})^{-\frac{1}{2}}+(S_{\X_{T}}^{*}S_{\X_{T}}-S_{\X_{S}}^{*}BS_{\X_{S}})(\lambda I+J_{T}^{*}J_{T})^{-\frac{1}{2}}\right](\lambda I+J_{T}^{*}J_{T})^{-\frac{1}{2}}.
\end{split}
\end{eqnarray}
Further,
 \begin{eqnarray}\nonumber
 \begin{split}
\|(\lambda I&+J_{T}^{*}J_{T})(\lambda I + S_{\X_{S}}^{*}BS_{\X_{S}})^{-1}\|_{\HK\rightarrow\HK}\\
&\le\Big[\|(\lambda I+J_{T}^{*}J_{T})^{-\frac{1}{2}}(J_{T}^{*}J_{T}-S_{\X_{T}}^{*}S_{\X_{T}})\|^{2}_{\HK\rightarrow\HK}\\
&+2\|(S_{\X_{T}}^{*}S_{\X_{T}}-S_{\X_{S}}^{*}BS_{\X_{S}})(\lambda I+J_{T}^{*}J_{T})^{-\frac{1}{2}}\|_{\HK\rightarrow\HK}\|(\lambda I+J_{T}^{*}J_{T})^{-\frac{1}{2}}(J_{T}^{*}J_{T}-S_{\X_{T}}^{*}S_{\X_{T}})\|_{\HK\rightarrow\HK}\\
&+\|(\lambda I+J_{T}^{*}J_{T})^{-\frac{1}{2}}(S_{\X_{T}}^{*}S_{\X_{T}}-S_{\X_{S}}^{*}BS_{\X_{S}})\|^{2}_{\HK\rightarrow\HK}\Big]\| (\lambda I + S_{\X_{S}}^{*}BS_{\X_{S}})^{-1}\|_{\HK\rightarrow\HK}\\
&+\Big[\|J_{T}^{*}J_{T}-S_{\X_{T}}^{*}S_{\X_{T}})(\lambda I+J_{T}^{*}J_{T})^{-\frac{1}{2}}\|_{\HK\rightarrow\HK}+\|(S_{\X_{T}}^{*}S_{\X_{T}}-S_{\X_{S}}^{*}BS_{\X_{S}})(\lambda I+J_{T}^{*}J_{T})^{-\frac{1}{2}}\|_{\HK\rightarrow\HK}\Big]\\
&\times\|(\lambda I+J_{T}^{*}J_{T})^{-\frac{1}{2}}\|_{\HK\rightarrow\HK}+1.
\end{split}
\end{eqnarray}
Next, applying Lemma \ref{my:lem2}, (\ref{bound:2}), and (\ref{eq:g}) with probability at least $1-\delta$, we have
 \begin{eqnarray}\nonumber
 \begin{split}
&\|(\lambda I+J_{T}^{*}J_{T})(\lambda I + S_{\X_{S}}^{*}BS_{\X_{S}})^{-1}\|_{\HK\rightarrow\HK}\\
&\le C\left[\left(\mathcal{B}_{m,\lambda}+\sqrt{\mathcal{N}_{\infty}(\lambda)}\left(\frac{1}{\sqrt{n}}+\frac{1}{\sqrt{m}}\right)\right)^{2}\right]\frac{1}{\lambda}\log^{2}\frac{2}{\delta}
+ C\left[\mathcal{B}_{m,\lambda}+\sqrt{\mathcal{N}_{\infty}(\lambda)}\left(\frac{1}{\sqrt{n}}+\frac{1}{\sqrt{m}}\right)\right]\frac{1}{\sqrt{\lambda}}\log\frac{2}{\delta}+1\\
&\le C \left[\left(\frac{\mathcal{B}_{m,\lambda}\log\frac{2}{\delta}}{\sqrt{\lambda}}+\sqrt{\frac{\mathcal{N}_{\infty}(\lambda)}{\lambda}}\left(\frac{1}{\sqrt{n}}+\frac{1}{\sqrt{m}}\right)\log\frac{2}{\delta}+\frac{1}{2}\right)^{2}\right]+\frac{3}{4}.
\end{split}
\end{eqnarray}

Further, using Proposition \ref{Cordes} we derive to the statement of the lemma.
\\
\\
{\bf{Proof of Lemma \ref{my:lem3}.}} \\
To prove the inequality (\ref{eq_my3}), 
 let consider the map $\xi(x,y)\colon X\times Y \rightarrow \HK$ defined by
$$\xi(x,y)= \left(\lambda I+J^{\ast}_{T}J_{T}\right)^{-\frac{1}{2}}\K(\cdot,x)\beta(x)y.$$
it is clear that
$$
\|\xi\|_{\HK}=\|\left(\lambda I+J^{\ast}_{T}J_{T}\right)^{-\frac{1}{2}}\K(\cdot,x)\beta(x) y\|_{\HK}\le \beta_0 y_0\sqrt{\mathcal{N}_x(\lambda)},
$$
Moreover, for $p(x,y)=\rho(y|x)\rho_{S}(x)$
 \begin{eqnarray}\nonumber
\begin{split}
\mathbb{E}\xi&=\int_{X\times Y}\left(\lambda I+J^{\ast}_{T}J_{T}\right)^{-\frac{1}{2}}\K(\cdot,x)\beta(x) y dp(x,y)= \int_{X} \left(\lambda I+J^{\ast}_{T}J_{T}\right)^{-\frac{1}{2}}\K(\cdot,x)\int_{Y} y d\rho(y|x)\beta(x)d\rho_{S}(x)\\
&= \int_{X} \left(\lambda I+J^{\ast}_{T}J_{T}\right)^{-\frac{1}{2}}\K(\cdot,x)\int_{Y} y d\rho(y|x)d\rho_{T}(x)= \int_{X} \left(\lambda I+J^{\ast}_{T}J_{T}\right)^{-\frac{1}{2}}\K(\cdot,x)f(x)d\rho_{T}(x)\\
&= \left(\lambda I+J^{\ast}_{T}J_{T}\right)^{-\frac{1}{2}}J_{T}^{\ast}f.
\end{split}
\end{eqnarray}
In addition, 
 \begin{eqnarray}\nonumber
\begin{split}
\mathbb{E}\|\xi\|^2_{\HK}&=\int_{X\times Y}\|\left(\lambda I+J^{\ast}_{T}J_{T}\right)^{-\frac{1}{2}}\K(\cdot,x)\beta(x) y\|^2_{\HK} dp(x,y)\\
=&\int_{X}\|\left(\lambda I+J^{\ast}_{T}J_{T}\right)^{-\frac{1}{2}}\K(\cdot,x)\|^2_{\HK}\int_{Y}|y|^{2}|\beta(x)|^2d\rho(y|x)d\rho_S(x)\\
&\le \beta_0 y_0^2\int_{X}\|\left(\lambda I+J^{\ast}_{T}J_{T}\right)^{-\frac{1}{2}}\K(\cdot,x)\|^2_{\HK}d\rho_T(x)\le \beta_0 y_0^2(\sqrt{\mathcal{N}_{\infty}(\lambda)})^2.
\end{split}
\end{eqnarray}
Note that 
$$
\frac{1}{n}\sum_{i=1}^{n}\,\xi(x_i,y_i)=\left(\lambda I+J^{\ast}_{T}J_{T}\right)^{-\frac{1}{2}}\frac{1}{n}\sum_{i=1}^{n}\K(\cdot,x_i)\beta(x_i) y_i=\left(\lambda I+J^{\ast}_{T}J_{T}\right)^{-\frac{1}{2}}S_{X_S}^{\ast}B\overline{y}.
$$
Hence, by Proposition \ref{concentrat}, with probability at least $1-\delta$ we get
 \begin{eqnarray}\nonumber
\begin{split}
\|(\lambda I&+J^{\ast}_{T}J_{T})^{-\frac{1}{2}}
\left [J_{T}^{\ast}f_q-S_{X_S}^{\ast} B\overline{y}\right]\|_{\HK}\\
&=C\|\frac{1}{n}\sum_{i=1}^{n}\,\xi(x_i,y_i)-\mathbb{E}\xi\|_{\HK}\le C \left(\frac{\sqrt{\mathcal{N}_{x}(\lambda)}}{n}+\sqrt{\frac{\mathcal{N}_{\infty}(\lambda)}{n}}\right)\log\frac{1}{\delta}\\
&\le C \sqrt{\mathcal{N}_{\infty}(\lambda)}\frac{1}{\sqrt{n}}\log\frac{1}{\delta}.
\end{split}
\end{eqnarray}
\\
\\
{\bf{Proof of Lemma \ref{my:lem4}.}} \\
We start with a decomposition of the left-hand side of (\ref{sig_22})
\begin{eqnarray}\nonumber
\begin{split}
 \| \left(\lambda I+P_{\z^{\nu}}J_{T}^{\ast}J_{T}P_{\z^{\nu}}\right)^{-1}P_{\z^{\nu}}J_{T}^{\ast}f_q\|_{\HK} 
&\le \|\PZnu J_{T}^{\ast}(\lambda I+J_{T}P_{\z^{\nu}}J_{T}^{\ast})^{-1}
\varphi(J_{T}P_{\z^{\nu}}J_{T}^{\ast})\mu_q\|_{\HK}\\
&+\|\PZnu J_{T}^{\ast}(\lambda I+J_{T}P_{\z^{\nu}}J_{T}^{\ast})^{-1}[\varphi(J_{T}J_{T}^{\ast})-\varphi(J_{T}P_{\z^{\nu}}J_{T}^{\ast})]\mu_q\|_{\HK}.
 \end{split}
\end{eqnarray}
Keeping in mind that $\varphi$ is operator monotone and applying  (\ref{qualific_root}) and (\ref{eq:g}), we get
 \begin{eqnarray}\nonumber
\begin{split}
 \| \left(\lambda I+P_{\z^{\nu}}J_{T}^{\ast}J_{T}P_{\z^{\nu}}\right)^{-1}P_{\z^{\nu}}J_{T}^{\ast}f_q\|_{\HK} \le \frac{C}{\sqrt{\lambda}}\varphi(\lambda)+\frac{C}{\sqrt{\lambda}}\varphi(\|J_T(I-\PZnu)\|^{2}_{\HK\rightarrow L_{2,\rho_T}})
 \end{split}
\end{eqnarray}
Since the function $\varphi(t)$ is covered by the qualification $p=1$, for any $C>1$ we have
\begin{equation}\label{eq:qual}
\frac{t}{\varphi(t)} \le \frac{Ct}{\varphi(Ct)} \quad
\Rightarrow \qquad \varphi(Ct)\le C \varphi(t).\nonumber
\end{equation}
This together with (\ref{eq:prob1}) completes the proof of the lemma.
\\
\\
{\bf{Proof of Lemma \ref{my:lem5}.}} \\
To prove the inequality (\ref{eq_my5}), first, note that $AB^{-1}=(A-B)A^{-1}(A-B)B^{-1}+(A-B)A^{-1}+I,$ then
\begin{eqnarray}\label{dec:2}
 \begin{split}
(\lambda I&+J_{T}^{*}J_{T})(\lambda I + S_{\X_{S}}^{*}B_{M,N}S_{\X_{S}})^{-1}\\
&=(J_{T}^{*}J_{T}-S_{\X_{S}}^{*}B_{M,N}S_{\X_{S}})(\lambda I+J_{T}^{*}J_{T})^{-1}(J_{T}^{*}J_{T}-S_{\X_{S}}^{*}B_{M,N}S_{\X_{S}})(\lambda I + S_{\X_{S}}^{*}B_{M,N}S_{\X_{S}})^{-1}\\
&+(J_{T}^{*}J_{T}-S_{\X_{S}}^{*}B_{M,N}S_{\X_{S}})(\lambda I + J_{T}^{*}J_{T})^{-1}+I.
\end{split}
\end{eqnarray}
Now, let's decompose the first term in (\ref{dec:2}):
\begin{eqnarray}\nonumber 
 \begin{split}
&(J_{T}^{*}J_{T}-S_{\X_{S}}^{*}B_{M,N}S_{\X_{S}})(\lambda I+J_{T}^{*}J_{T})^{-1}(J_{T}^{*}J_{T}-S_{\X_{S}}^{*}B_{M,N}S_{\X_{S}})(\lambda I + S_{\X_{S}}^{*}B_{M,N}S_{\X_{S}})^{-1}\\
&=\left[(J_{T}^{*}J_{T}-S_{\X_{S}}^{*}BS_{\X_{S}})(\lambda I+J_{T}^{*}J_{T})^{-\frac{1}{2}}+(S_{\X_{S}}^{*}BS_{\X_{S}}-S_{\X_{S}}^{*}B_{M,N}S_{\X_{S}})(\lambda I+J_{T}^{*}J_{T})^{-\frac{1}{2}}\right]\\
&\times\left[(\lambda I+J_{T}^{*}J_{T})^{-\frac{1}{2}}(J_{T}^{*}J_{T}-S_{\X_{S}}^{*}BS_{\X_{S}})+(\lambda I+J_{T}^{*}J_{T})^{-\frac{1}{2}}(S_{\X_{S}}^{*}BS_{\X_{S}}-S_{\X_{S}}^{*}B_{M,N}S_{\X_{S}})\right](\lambda I + S_{\X_{S}}^{*}BS_{\X_{S}})^{-1}.
\end{split}
\end{eqnarray}
For the second term  in (\ref{dec:2}) we have
\begin{eqnarray}\nonumber
 \begin{split}
 &(J_{T}^{*}J_{T}-S_{\X_{S}}^{*}B_{M,N}S_{\X_{S}})(\lambda I + J_{T}^{*}J_{T})^{-1}\\
&=\left[J_{T}^{*}J_{T}-S_{\X_{S}}^{*}BS_{\X_{S}})(\lambda I+J_{T}^{*}J_{T})^{-\frac{1}{2}}+(S_{\X_{S}}^{*}BS_{\X_{S}}-S_{\X_{S}}^{*}B_{M,N}S_{\X_{S}})(\lambda I+J_{T}^{*}J_{T})^{-\frac{1}{2}}\right](\lambda I+J_{T}^{*}J_{T})^{-\frac{1}{2}}.
\end{split}
\end{eqnarray}
Further,
 \begin{eqnarray}\nonumber
 \begin{split}
\|(\lambda I&+J_{T}^{*}J_{T})(\lambda I + S_{\X_{S}}^{*}BS_{\X_{S}})^{-1}\|_{\HK\rightarrow\HK}\\
&\le\Big[\|(\lambda I+J_{T}^{*}J_{T})^{-\frac{1}{2}}(J_{T}^{*}J_{T}-S_{\X_{S}}^{*}BS_{\X_{S}})\|^{2}_{\HK\rightarrow\HK}\\
&+2\|(S_{\X_{S}}^{*}BS_{\X_{S}}-S_{\X_{S}}^{*}B_{M,N}S_{\X_{S}})(\lambda I+J_{T}^{*}J_{T})^{-\frac{1}{2}}\|_{\HK\rightarrow\HK}\|(\lambda I+J_{T}^{*}J_{T})^{-\frac{1}{2}}(J_{T}^{*}J_{T}-S_{\X_{S}}^{*}BS_{\X_{S}})\|_{\HK\rightarrow\HK}\\
&+\|(\lambda I+J_{T}^{*}J_{T})^{-\frac{1}{2}}(S_{\X_{S}}^{*}BS_{\X_{S}}-S_{\X_{S}}^{*}B_{M,N}S_{\X_{S}})\|^{2}_{\HK\rightarrow\HK}\Big]\| (\lambda I + S_{\X_{S}}^{*}BS_{\X_{S}})^{-1}\|_{\HK\rightarrow\HK}\\
&+\Big[\|J_{T}^{*}J_{T}-S_{\X_{S}}^{*}BS_{\X_{S}})(\lambda I+J_{T}^{*}J_{T})^{-\frac{1}{2}}\|_{\HK\rightarrow\HK}+\|(S_{\X_{S}}^{*}BS_{\X_{S}}-S_{\X_{S}}^{*}B_{M,N}S_{\X_{S}})(\lambda I+J_{T}^{*}J_{T})^{-\frac{1}{2}}\|_{\HK\rightarrow\HK}\Big]\\
&\times\|(\lambda I+J_{T}^{*}J_{T})^{-\frac{1}{2}}\|_{\HK\rightarrow\HK}+1\\
&=\Big(\|(\lambda I+J_{T}^{*}J_{T})^{-\frac{1}{2}}(J_{T}^{*}J_{T}-S_{\X_{S}}^{*}BS_{\X_{S}})\|_{\HK\rightarrow\HK}
+\|(\lambda I+J_{T}^{*}J_{T})^{-\frac{1}{2}}(S_{\X_{S}}^{*}BS_{\X_{S}}-S_{\X_{S}}^{*}B_{M,N}S_{\X_{S}})\|^{2}_{\HK\rightarrow\HK}\Big)^{2}\\
&\times\| (\lambda I + S_{\X_{S}}^{*}BS_{\X_{S}})^{-1}\|_{\HK\rightarrow\HK}
+\Big[\|J_{T}^{*}J_{T}-S_{\X_{S}}^{*}BS_{\X_{S}})(\lambda I+J_{T}^{*}J_{T})^{-\frac{1}{2}}\|_{\HK\rightarrow\HK}\\
&+\|(S_{\X_{S}}^{*}BS_{\X_{S}}-S_{\X_{S}}^{*}B_{M,N}S_{\X_{S}})(\lambda I+J_{T}^{*}J_{T})^{-\frac{1}{2}}\|_{\HK\rightarrow\HK}\Big]
\|(\lambda I+J_{T}^{*}J_{T})^{-\frac{1}{2}}\|_{\HK\rightarrow\HK}+1.
\end{split}
\end{eqnarray}
From the proof of Lemma \ref{my:lem1} it follows that
\begin{eqnarray}\label{eq_nB}
 \begin{split}
\|(\lambda I+J_{T}^{*}J_{T})^{-\frac{1}{2}}(J_{T}^{*}J_{T}-S_{\X_{S}}^{*}BS_{\X_{S}})\|_{\HK\rightarrow\HK}\le \mathcal{B}_{m,\lambda}\log\frac{2}{\delta}+\sqrt{\mathcal{N}_{\infty}(\lambda)}\left(\frac{1}{\sqrt{n}}+\frac{1}{\sqrt{m}}\right).
  \end{split}
\end{eqnarray}

This together with Lemma \ref{my:lem2}, (\ref{bound:2}), (\ref{eq_n5}), and (\ref{eq:g})  implies that
 \begin{eqnarray}\nonumber
 \begin{split}
\|(\lambda I&+J_{T}^{*}J_{T})(\lambda I + S_{\X_{S}}^{*}B_{M,N}S_{\X_{S}})^{-1}\|_{\HK\rightarrow\HK}\\
&\le C\left(\mathcal{B}_{m,\lambda}+\sqrt{\mathcal{N}_{\infty}(\lambda)}\left(\frac{1}{\sqrt{n}}+\frac{1}{\sqrt{m}}\right)+\sqrt{\mathcal{N}_{\infty}(\lambda)}E_{\beta}\right)^{2}\frac{1}{\lambda}\log^{2}\frac{2}{\delta}\\
&+ C\left[\mathcal{B}_{m,\lambda}+\sqrt{\mathcal{N}_{\infty}(\lambda)}\left(\frac{1}{\sqrt{n}}+\frac{1}{\sqrt{m}}\right)+\sqrt{\mathcal{N}_{\infty}(\lambda)}E_{\beta}\right]\frac{1}{\sqrt{\lambda}}\log\frac{2}{\delta}+1\\
&\le C \left(\frac{\mathcal{B}_{m,\lambda}\log\frac{2}{\delta}}{\sqrt{\lambda}}+\left[E_{\beta}+\left(\frac{1}{\sqrt{n}}+\frac{1}{\sqrt{m}}\right)\right]\sqrt{\frac{\mathcal{N}_{\infty}(\lambda)}{\lambda}}\log\frac{2}{\delta}+\frac{1}{2}\right)^{2}+\frac{3}{4}.
\end{split}
\end{eqnarray}

Finally, using Proposition \ref{Cordes} we derive to the statement of the lemma.
\newpage

\renewcommand{\thetheorem}{B.\arabic{theorem}}
\renewcommand{\theproposition}{B.\arabic{proposition}}
\renewcommand{\thedefinition}{B.\arabic{definition}}
\renewcommand{\thecorollary}{B.\arabic{corollary}}
\renewcommand{\thelemma}{B.\arabic{lemma}}
\renewcommand{\theremark}{B.\arabic{remark}}
\renewcommand{\theexample}{B.\arabic{example}}
\renewcommand{\theequation}{B.\arabic{equation}}
\section{Appendix. Proof of Proposition \ref{resultHK} }\label{ap:b}
\begin{proof}
We start with a decomposition
\begin{eqnarray}\label{err_decomp}
\begin{split}
f_{q}-J_{T}f_{\z,\z^{\nu}}^{\lambda_{m,n}}&=f_{q}-
J_{T}\left(\lambda I+P_{\z^{\nu}}S_{X_S}^{\ast} BS_{X_S}P_{\z^{\nu}}\right)^{-1}
P_{\z^{\nu}}S_{X_S}^{\ast} B\overline{y}
=\sigma_1+\sigma_2+\sigma_3,
\end{split}
\end{eqnarray}
where
\begin{eqnarray}
\begin{split}
\sigma_1&:=f_{q} -
J_T\left(\lambda I+P_{\z^{\nu}}J_{T}^{\ast}J_{T}P_{\z^{\nu}}\right)^{-1}P_{\z^{\nu}}J_{T}^{\ast}f_q,\\
\sigma_2&:=J_T\left(\lambda I+P_{\z^{\nu}}J_{T}^{\ast}J_{T}P_{\z^{\nu}}\right)^{-1}P_{\z^{\nu}}J_{T}^{\ast}f_q- J_{T}\left(\lambda I+P_{\z^{\nu}}S_{X_S}^{\ast} BS_{X_S}P_{\z^{\nu}}\right)^{-1}P_{\z^{\nu}}J_{T}^{\ast}f_q,\nonumber\\
\sigma_3&:=J_{T}\left(\lambda I+P_{\z^{\nu}}S_{X_S}^{\ast} BS_{X_S}P_{\z^{\nu}}\right)^{-1}P_{\z^{\nu}}J_{T}^{\ast}f_q-
J_{T}\left(\lambda I+P_{\z^{\nu}}S_{X_S}^{\ast} BS_{X_S}P_{\z^{\nu}}\right)^{-1}
P_{\z^{\nu}}S_{X_S}^{\ast} B\overline{y} .\nonumber
\end{split}
\end{eqnarray}
 We estimate the norm of each terms
$
\sigma_i,$ $i=\overline{1,3}.$
Note that the norms of $\sigma_1$  and $\overline{\sigma}_1$ are equal (see (\ref{ovsig1_decomp}), (\ref{sigma1})). Thus, 
\begin{eqnarray}\label{sig1}
\begin{split}
\|\sigma_1\|_{L_{2,\rho_{T}}}=\|\overline{\sigma}_1\|_{L_{2,\rho_{T}}}&\le 
C\varphi(\lambda).
\end{split}
\end{eqnarray}

Now, we at the point to estimate the norm of $\sigma_2$. Recall that
\begin{eqnarray}\nonumber
\begin{split}
\sigma_2&:=J_T\left(\left(\lambda I+P_{\z^{\nu}}J_{T}^{\ast}J_{T}P_{\z^{\nu}}\right)^{-1}-\left(\lambda I+P_{\z^{\nu}}S_{X_S}^{\ast} BS_{X_S}P_{\z^{\nu}}\right)^{-1}\right)P_{\z^{\nu}}J_{T}^{\ast}f_q\\
&=J_T\left(\lambda I+P_{\z^{\nu}}S_{X_S}^{\ast} BS_{X_S}P_{\z^{\nu}}\right)^{-1}\PZnu\left [S_{X_S}^{\ast} BS_{X_S}-J_{T}^{\ast}J_{T}\right] \PZnu \left(\lambda I+P_{\z^{\nu}}J_{T}^{\ast}J_{T}P_{\z^{\nu}}\right)^{-1}P_{\z^{\nu}}J_{T}^{\ast}f_q.
\end{split}
\end{eqnarray}
From this we have

\begin{eqnarray}\nonumber
\begin{split}
\|\sigma_{2}\|_{L_{2,\rho_{T}}}&\le\|J_{T}\left(\lambda I+J^{\ast}_{T}J_{T}\right)^{-\frac{1}{2}}\|_{\HK\rightarrow L_{2,\rho_{T}}}\|\left(\lambda I+J^{\ast}_{T}J_{T}\right)^{\frac{1}{2}}\left(\lambda I+S_{X_S}^{\ast} BS_{X_S}\right)^{-\frac{1}{2}}\|_{\HK\rightarrow\HK}\\
&\times \|\left(\lambda I+S_{X_S}^{\ast} BS_{X_S}\right)^{\frac{1}{2}}\left(\lambda I+P_{\z^{\nu}}S_{X_S}^{\ast} BS_{X_S}P_{\z^{\nu}}\right)^{-1}\PZnu\left(\lambda I+S_{X_S}^{\ast} BS_{X_S}\right)^{\frac{1}{2}}\|_{\HK\rightarrow\HK}\\
&\times \|\left(\lambda I+S_{X_S}^{\ast} BS_{X_S}\right)^{-\frac{1}{2}}\left(\lambda I+J^{\ast}_{T}J_{T}\right)^{\frac{1}{2}}\|_{\HK\rightarrow\HK}\\
&\times\|\left(\lambda I+J^{\ast}_{T}J_{T}\right)^{-\frac{1}{2}}
\left [S_{X_S}^{\ast} BS_{X_S}-J_{T}^{\ast}J_{T}\right]\|_{\HK\rightarrow\HK}\\
&\times \|\left(\lambda I+P_{\z^{\nu}}J_{T}^{\ast}J_{T}P_{\z^{\nu}}\right)^{-1}P_{\z^{\nu}}J_{T}^{\ast}f_q\|_{\HK}.
\end{split}
\end{eqnarray}
Note that
\begin{eqnarray}\nonumber
\begin{split}
&\|\left(\lambda I+J^{\ast}_{T}J_{T}\right)^{-\frac{1}{2}}
\left [S_{X_S}^{\ast} BS_{X_S}-J_{T}^{\ast}J_{T}\right]\|_{\HK\rightarrow\HK}\le
\|\left(\lambda I+J^{\ast}_{T}J_{T}\right)^{-\frac{1}{2}}\left [S_{X_S}^{\ast} BS_{X_S}-J_{T}^{\ast}J_{T}\right]\|_{\HK\rightarrow\HK} \\
&\le \|\left(\lambda I+J^{\ast}_{T}J_{T}\right)^{-\frac{1}{2}}\left [S_{X_S}^{\ast} BS_{X_S}-S_{\X_T}^{\ast}S_{\X_T}\right]\|_{\HK\rightarrow\HK}+\|\left(\lambda I+J^{\ast}_{T}J_{T}\right)^{-\frac{1}{2}}\left[S_{\X_T}^{\ast}S_{\X_T}-J_{T}^{\ast}J_{T}\right]\|_{\HK\rightarrow\HK}.
\end{split}
\end{eqnarray}
Applying Lemma \ref{my:lem2} and (\ref{bound:1}), we get
\begin{eqnarray}\label{add_s2}
\begin{split}
\|\left(\lambda I+J^{\ast}_{T}J_{T}\right)^{-\frac{1}{2}}
\left [S_{X_S}^{\ast} BS_{X_S}-J_{T}^{\ast}J_{T}\right]\|_{\HK\rightarrow\HK}\le
 C\left(\sqrt{\mathcal{N}_{\infty}(\lambda)}\left(\frac{1}{\sqrt{n}}+\frac{1}{\sqrt{m}}\right)+\mathcal{B}_{m,\lambda}\right)\log\frac{2}{\delta}.
\end{split}
\end{eqnarray}

This together with Lemma \ref{lem:polar}, (\ref{eq:prob20}), Lemma \ref{my:lem1}, (\ref{bound:2}), Proposition \ref{Cordes}, and Lemma \ref{my:lem4} implies that 
\begin{eqnarray}\label{sigma2}
\begin{split}
&\|\sigma_2\|_{L_{2,\rho_T}} \le C\left[ \left(\frac{\mathcal{B}_{m,\lambda}\log\frac{2}{\delta}}{\sqrt{\lambda}}+\sqrt{\frac{\mathcal{N}_{\infty}(\lambda)}{\lambda}}\left(\frac{1}{\sqrt{n}}+\frac{1}{\sqrt{m}}\right)\log\frac{2}{\delta}+\frac{1}{2}\right)^{2}+\frac{3}{4}\right]\frac{\mathcal{B}_{m,\lambda}}{\sqrt{\lambda}}\varphi(\lambda)\log\frac{2}{\delta}\\
&+ C\left[ \left(\frac{\mathcal{B}_{m,\lambda}\log\frac{2}{\delta}}{\sqrt{\lambda}}+\sqrt{\frac{\mathcal{N}_{\infty}(\lambda)}{\lambda}}\left(\frac{1}{\sqrt{n}}+\frac{1}{\sqrt{m}}\right)\log\frac{2}{\delta}+\frac{1}{2}\right)^{2}+\frac{3}{4}\right]\sqrt{\frac{\mathcal{N}_{\infty}(\lambda)}{\lambda}}\left(\frac{1}{\sqrt{n}}+\frac{1}{\sqrt{m}}\right)\varphi(\lambda)\log\frac{2}{\delta}.
\end{split}
\end{eqnarray}
The rest of the proof is about the estimation of $\|\sigma_3\|_{L_{2,\rho_T}}.$ First, we rewrite $\|\sigma_3\|_{L_{2,\rho_T}}$ as follows
 \begin{eqnarray}\nonumber
\begin{split}
\|\sigma_3\|_{L_{2,\rho_T}}&=\|J_{T}\left(\lambda I+J^{\ast}_{T}J_{T}\right)^{-\frac{1}{2}}\|_{\HK\rightarrow L_{2,\rho_{T}}}\|\left(\lambda I+J^{\ast}_{T}J_{T}\right)^{\frac{1}{2}}\left(\lambda I+S_{X_S}^{\ast} BS_{X_S}\right)^{-\frac{1}{2}}\|_{\HK\rightarrow\HK}\\
&\times \|\left(\lambda I+S_{X_S}^{\ast} BS_{X_S}\right)^{\frac{1}{2}}\left(\lambda I+P_{\z^{\nu}}S_{X_S}^{\ast} BS_{X_S}P_{\z^{\nu}}\right)^{-1}\PZnu\left(\lambda I+S_{X_S}^{\ast} BS_{X_S}\right)^{\frac{1}{2}}\|_{\HK\rightarrow\HK}\\
&\times \|\left(\lambda I+S_{X_S}^{\ast} BS_{X_S}\right)^{-\frac{1}{2}}\left(\lambda I+J^{\ast}_{T}J_{T}\right)^{\frac{1}{2}}\|_{\HK\rightarrow\HK}\\
&\times\|\left(\lambda I+J^{\ast}_{T}J_{T}\right)^{-\frac{1}{2}}
\left [J_{T}^{\ast}f_q-S_{X_S}^{\ast} B\overline{y}\right]\|_{\HK}.
 \end{split}
\end{eqnarray}
By means of Lemma \ref{lem:polar}, (\ref{eq:prob2}), (\ref{bound:2}), Proposition \ref{Cordes}, and Lemmas  \ref{my:lem1}, \ref{my:lem3}, we obtain
 \begin{eqnarray}\label{sigma3}
\begin{split}
&\|\sigma_3\|_{L_{2,\rho_T}}\le C\left[ \left(\frac{\mathcal{B}_{m,\lambda}\log\frac{2}{\delta}}{\sqrt{\lambda}}+\sqrt{\frac{\mathcal{N}_{\infty}(\lambda)}{\lambda}}\left(\frac{1}{\sqrt{n}}+\frac{1}{\sqrt{m}}\right)\log\frac{2}{\delta}+\frac{1}{2}\right)^{2}+\frac{3}{4}\right]
 \sqrt{\mathcal{N}_{\infty}(\lambda)}\frac{1}{\sqrt{n}}\log\frac{2}{\delta}.
\end{split}
\end{eqnarray}

Thus, summing up (\ref{sig1}), (\ref{sigma2}), and (\ref{sigma3}) with probability at least $1-\delta$, we finally obtain
 \begin{eqnarray}\label{err_gen}\nonumber
\begin{split}
&\|f_q-J_Tf_{\z,\Znu}^{\lambda_{m,n}}\|_{L_{2},\rho_{T}} \le C\varphi(\lambda)+C\left[ \left(\frac{\mathcal{B}_{m,\lambda}\log\frac{2}{\delta}}{\sqrt{\lambda}}+\sqrt{\frac{\mathcal{N}_{\infty}(\lambda)}{\lambda}}\left(\frac{1}{\sqrt{n}}+\frac{1}{\sqrt{m}}\right)\log\frac{2}{\delta}+\frac{1}{2}\right)^{2}+\frac{3}{4}\right]\frac{\mathcal{B}_{m,\lambda}}{\sqrt{\lambda}}\varphi(\lambda)\log\frac{2}{\delta}\\
&+ C\left[ \left(\frac{\mathcal{B}_{m,\lambda}\log\frac{2}{\delta}}{\sqrt{\lambda}}+\sqrt{\frac{\mathcal{N}_{\infty}(\lambda)}{\lambda}}\left(\frac{1}{\sqrt{n}}+\frac{1}{\sqrt{m}}\right)\log\frac{2}{\delta}+\frac{1}{2}\right)^{2}+\frac{3}{4}\right]\sqrt{\frac{\mathcal{N}_{\infty}(\lambda)}{\lambda}}\left(\frac{1}{\sqrt{n}}+\frac{1}{\sqrt{m}}\right)\varphi(\lambda)\log\frac{2}{\delta}\\
&+C\left[ \left(\frac{\mathcal{B}_{m,\lambda}\log\frac{2}{\delta}}{\sqrt{\lambda}}+\sqrt{\frac{\mathcal{N}_{\infty}(\lambda)}{\lambda}}\left(\frac{1}{\sqrt{n}}+\frac{1}{\sqrt{m}}\right)\log\frac{2}{\delta}+\frac{1}{2}\right)^{2}+\frac{3}{4}\right]
 \sqrt{\mathcal{N}_{\infty}(\lambda)}\frac{1}{\sqrt{n}}\log\frac{2}{\delta}.
\end{split}
\end{eqnarray}

\end{proof}
\end{document}